\documentclass[]{article}
\usepackage{preprint}

\newcommand{\subto}{\mathrm{subject\ to}}

\def\1{\bm{1}}

\def\rmX{{\mathbf{X}}}

\DeclareMathAlphabet{\mathsfit}{\encodingdefault}{\sfdefault}{m}{sl}
\SetMathAlphabet{\mathsfit}{bold}{\encodingdefault}{\sfdefault}{bx}{n}

\def\gG{{\mathcal{G}}}

\def\gN{{\mathcal{N}}}

\def\gP{{\mathcal{P}}}

\def\gS{{\mathcal{S}}}

\def\gY{{\mathcal{Y}}}
\def\gZ{{\mathcal{Z}}}

\def\sR{{\mathbb{R}}}

\newcommand{\E}{\mathbb{E}}

\newcommand{\R}{\mathbb{R}}

\DeclareMathOperator*{\argmin}{arg\,min}

\DeclareMathOperator{\Tr}{Tr}

\newcommand{\Norm}[1]{\ensuremath{\lVert #1 \rVert}}                  
\newcommand{\NormI}[1]{\ensuremath{\lVert #1 \rVert}_1}               
\newcommand{\NormII}[1]{\ensuremath{\lVert #1 \rVert}_2}              

\newcommand{\InNorm}[1]{{\left\vert\kern-0.2ex\left\vert\kern-0.2ex\left\vert #1 
    \right\vert\kern-0.2ex\right\vert\kern-0.2ex\right\vert}}                    

\newcommand{\InNormII}[1]{{\left\vert\kern-0.2ex\left\vert\kern-0.2ex\left\vert #1 
    \right\vert\kern-0.2ex\right\vert\kern-0.2ex\right\vert}_2}                    

\newcommand{\InNormInfty}[1]{{\left\vert\kern-0.2ex\left\vert\kern-0.2ex\left\vert #1 
    \right\vert\kern-0.2ex\right\vert\kern-0.2ex\right\vert}_{\infty}}           

\newcommand{\Abs}[1]{\ensuremath{\lvert #1 \rvert}}                              
\newcommand{\AAbs}[1]{\ensuremath{\left \lvert #1 \right \rvert}}                              
\newcommand{\iid}{i.i.d.~}                                                        

\newcommand{\defeq}{\overset{\mathrm{def}}{=}}                                   

\DeclareMathOperator*{\union}{\cup}

\newcommand{\Best}[1]{\color{black}{\mathbf{#1}}}

\newtheorem{definition}{Definition}
\newtheorem{proposition}{Proposition}

\newtheorem{lemma}{Lemma}

\newtheorem{theorem}{Theorem}
\newtheorem{remark}{Remark}
\newtheorem{corollary}{Corollary}
\newtheorem{condition}{Condition}

\newtheorem{example}{Example}

\newcommand{\apptitle}[1]{
\def\toptitlebar{
	\hrule height4pt
	\vskip .25in}

\def\bottomtitlebar{
	\vskip .25in
	\hrule height1pt
	\vskip .25in}

\thispagestyle{plain}
\hsize\textwidth
\linewidth\hsize \toptitlebar {\centering
{\Large\bf SUPPLEMENTARY MATERIAL \\ #1 \par}}
\vspace{-0.1in} \bottomtitlebar
}

\newcommand{\Erdos}{Erdős} 
\newcommand{\Renyi}{Rényi }

\newcommand{\Z}{\mathcal{Z}}
\newcommand{\G}{\mathcal{G}}
\newcommand{\Y}{\mathcal{Y}}

\newcommand{\PA}{\text{PA}}
\newcommand{\subjto}{\text{subject to}}
\newcommand{\topoRandom}{\mathrm{Random\_Topo}}
\newcommand{\topoNotears}{\mathrm{Notears\_Topo}}

\newcommand{\topoGreedyRandom}{\mathrm{Random\_Greedy\_Topo}}
\newcommand{\MLP}{\textsf{MLP}}

\newcommand{\ttheta}{\tilde{\theta}}
\newcommand{\notears}{\textsc{Notears }}
\newcommand{\dagma}{\textsc{Dagma }}

\renewcommand{\gG}{G}
\newcommand{\opt}{\theta^{*}}
\newcommand{\ordopt}{\opt_{\pi}}
\newcommand{\Ys}{s_{\text{small}}}
\newcommand{\Yl}{s_{\text{large}}}
\newcommand{\Ordopt}{\Theta^{*}_\pi}
\def\truedag{W^{\dagger}}
\DeclareGraphicsExtensions{.pdf,.png,.jpg,.eps}
\graphicspath{{plot/}}
\allowdisplaybreaks

\title{Optimizing NOTEARS Objectives via\\ Topological Swaps}
\author{
	Chang Deng$^\dag$ \qquad
	Kevin Bello$^{\dag\ddag}$ \qquad
	Bryon Aragam$^\dag$\qquad
	Pradeep Ravikumar$^\ddag$\\	
	$^\dag$Booth School of Business, The University of Chicago, Chicago, IL 60637\\
	$^{\ddag}$Machine Learning Department, Carnegie Mellon University, Pittsburgh, PA 15213
}
\begin{document}

\maketitle

\begin{abstract}
  Recently, an intriguing class of non-convex optimization problems has emerged in the context of learning directed acyclic graphs (DAGs). These problems involve minimizing a given loss or score function, subject to a non-convex continuous constraint that penalizes the presence of cycles in a graph. In this work, we delve into the optimization challenges associated with this class of non-convex programs. To address these challenges, we propose a bi-level algorithm that leverages the non-convex constraint in a novel way. The outer level of the algorithm optimizes over topological orders by iteratively swapping pairs of nodes within the topological order of a DAG. A key innovation of our approach is the development of an effective method for generating a set of candidate swapping pairs for each iteration. At the inner level, given a topological order, we utilize off-the-shelf solvers that can handle linear constraints. The key advantage of our proposed algorithm is that it is guaranteed to find a local minimum or a KKT point under weaker conditions compared to previous work and finds solutions with lower scores. Extensive experiments demonstrate that our method outperforms state-of-the-art approaches in terms of achieving a better score. Additionally, our method can also be used as a post-processing algorithm to significantly improve the score of other algorithms. Code implementing the proposed method is available at \href{https://github.com/Duntrain/TOPO}{{https://github.com/duntrain/topo}}.
\end{abstract}

\section{Introduction}
\label{sec:introduction}

We study a class of constrained nonconvex optimization problems defined as follows:
\begin{align}
\label{eq:general_opt}
	\min_\Theta  Q(\Theta)\;\; \subto\;\; h(W(\Theta)) = 0,
\end{align}
where $\Theta \in \sR^l$ corresponds to all model parameters, and $W(\Theta) \in \sR^{d\times d}$ is a weighted adjacency matrix ---representing the structure of a directed graph of $d$ nodes---induced by $\Theta$.
Moreover, $Q: \sR^l \to \sR$ is a (possibly non-convex) differentiable function that we will refer to as the score or loss function; while $h: \sR^{d\times d} \to [0,\infty)$ is a nonnegative \textbf{non-convex} differentiable function that penalizes cycles in the weighted adjacency matrix $W(\Theta)$, and whose level set at zero corresponds to directed acyclic graphs (DAGs).

The class of problems \eqref{eq:general_opt} arose in the paper by \citet{xun2018} in the context of learning the underlying structure of a structural equation model (SEM), typically assumed to be a DAG.
In \citet{xun2018}, the challenges of combinatorial optimization were replaced by those of differentiable non-convex optimization.
While global optimality remains intractable in general, the key advantage of the class of problems \eqref{eq:general_opt} is that it admits the use of general purpose non-linear optimizers.
Due to the latter, several studies have built upon the work of \citet{xun2018}, usually by either proposing a new characterization of $h$ \citep[e.g.,][]{yu2019dag,bello2022}, or using different score functions $Q$ \citep[e.g.,][]{xun2018, zheng2020, Ignavier2020, yu2019dag, lachapelle2019gradient}.
All, however, with a clear lack of optimality guarantees.

Based on these formulations,
\citet{dennis2020} and \citet{ng2022convergence} studied some of the optimization-theoretic curiosities associated with this class of problems.
\citet{dennis2020} provides local optimality guarantees 
assuming \textit{linear models} and a \textit{convex} score $Q$, while \citet{ng2022convergence} studies the convergence challenges of \eqref{eq:general_opt}.
The focus of our work is studying optimality for the class of problems \eqref{eq:general_opt} in a more general setting, i.e., admitting a possibly non-convex score $Q$ and nonlinear models.
We pay close attention to the Karush-Kuhn-Tucker (KKT) optimality conditions, building upon similar results first studied in \citet{dennis2020}. The KKT conditions are known to be a necessary first-order characterization of optimal solutions under some regularity conditions, and form the backbone of nonlinear programming \citep{bertsekas1997nonlinear,boyd2004convex}.

More specifically, we show that by an equivalent reformulation of the KKT conditions, we can find \emph{better} solutions to \eqref{eq:general_opt} --- that is, KKT points and/or local minima with better (i.e. lower) score --- while also relaxing the conditions required in previous work. 
The key idea is to relate the KKT conditions to an optimal topological sort and leverage the fact that solving the continuous program for a fixed ordering is often tractable.
Although not every topological sort corresponds to a local minimum in the continuous formulation, we show that our method can indeed be rigorously interpreted as iteratively selecting better and better local minimizers until no improvement can be found. 
Our method also avoids explicitly enforcing the acyclicity constraint $h$, and instead uses the continuous characterization \emph{indirectly} via the KKT conditions.

\paragraph{Contributions.}
To this end, we make the following specific contributions:

\begin{enumerate}
    \item We propose a bi-level optimization algorithm, in which the outer level optimizes over topological orders and the inner level optimizes the score given a specific order. To optimize over orders, we use a novel technique for selecting candidate pairs of nodes to be swapped, which is described in detail in Section \ref{sec:order}. This approach involves iteratively swapping pairs of nodes within the topological order of a DAG, and utilizes the KKT conditions as a guide for determining which pairs to consider swapping. To optimize the score given a specific order, we utilize state-of-the-art solvers that are able to solve the problems to stationary points.
    \item We prove that our method searches between local minima and strictly decreases the score at each iteration (Section~\ref{sec:analysis}). 
    We furthermore show that our method provably finds local minimizers under strictly weaker conditions compared to previous work (Lemma~\ref{lemma:example}). In particular, we show that the concept of `irreducibility' introduced in \citet{dennis2020} is not necessary to ensure local optimality, and provide an explicit example as demonstration (Appendix~\ref{sec:prev_work}).
    \item We conduct a comprehensive set of experiments in multiple settings to evaluate the performance of our algorithm against state-of-the-art methods for solving problem~\eqref{eq:general_opt}. The results of our experiments, summarized in Section~\ref{sec:exp},  demonstrate 
    that our method is able to find minimizers with lower scores (compared to existing algorithms) that are guaranteed to be either local minima or KKT points.
\end{enumerate}

An attractive feature of our method is its flexibility as it can be used both as a standalone algorithm and as a post-processing step when provided with a pre-computed DAG as an initialization. Although the underlying optimization problem is nonconvex and plagued by poor local minima, our results demonstrate that it is still possible to discover suitable local minima with improved scores. This is a noteworthy achievement given that nonconvex problems of this nature are often considered challenging and difficult to optimize.

\section{Related Work}
\label{sec:related_work}

Most closely related to our work are methods that build on the non-convex continuous constrained formulation of \citet{xun2018}, \citep[e.g.,][]{yu2019dag, zheng2020, lachapelle2019gradient, Ignavier2020, zhu2020causal,romain2020bregman,bello2022}. 
In contrast to this previous work, our focus is on \emph{optimality conditions}, i.e. ensuring that we find a DAG that satisfies the KKT optimality conditions (in fact, it will be a local minimizer) of an equivalent formulation to that of \citet{xun2018}. Similar to our work, recent work \citep{dennis2020,ng2022convergence} has begun to study the optimization-theoretic aspects of this problem.
In contrast to \citet{dennis2020}, which is only guaranteed to return \emph{some} local minimizer, our method iteratively jumps from one local minimizer to another until a stopping criterion is met. The latter allows our method to seek out for more favorable local minimizers, that is, DAGs that attain lower scores. \citet{ng2022convergence} studies a different question, namely the convergence of methods for solving these problems.

Although our emphasis is on optimization, it is useful to provide some context from the graphical modeling literature as well.
Most algorithms for learning DAGs fall into two main categories: score-based methods that optimize a score function, and constraint-based methods that use independence tests.
Since the program \eqref{eq:general_opt} is modeled after traditional score-based methods, we only mention a few classical constraint-based algorithms such as: the PC algorithm \citep{spirtes1991}, a general algorithm that learns the Markov equivalence class; max-min parents and children \citep[MMPC,][]{tsamardinos2006}; and a variety of algorithms based on local Markov boundary search such as grow-shrink \citep[GS,][]{margaritis1999bayesian,margaritis2003learning} and incremental association \citep[IAMB,][]{tsamardinos2003algorithms}.

Score based methods assign a score to a candidate DAG structure based on how well it fits the observed data, and then attempts to find the highest scoring structure. Classical score functions include the log-likelihood based BIC and AIC scores as well as Bayesian scores under different parameter priors \citep{geiger2002}.
Other related work that study the Gaussian setting are given by \citet{aragam2015concave, ghoshal2017learning, ghoshal18}, and in the non-Gaussian case by \citet{loh2014high}.
On the side of approximate algorithms, notable methods include greedy search \citep{chickering2002}, order search \citep{marc2005,scanagatta2015learning,park2017bayesian}, and the LP-relaxation based method proposed by \citet{jaakkola10a}. There are also exact algorithms such as GOBNILP \citep{cussens2012} and bene \citep{silander2012}.

Another line of work \citep{marc2005,xiang2013lasso,raskutti2018learning,drton2018causal,ye2020optimizing,squires2020permutation,solus2021consistency,wang2021ordering} studies order-based methods which bear a superficial relationship to our algorithm, but it is worth emphasizing that none of them theoretically analyze optimization properties such as KKT theory, local optimality guarantees or apply to \textit{arbitrary} smooth losses. 
More specifically, \citep{ye2020optimizing} is restricted to log-likelihood based scores and \citep{raskutti2018learning,squires2020permutation,solus2021consistency} require faithfulness (related) assumptions. \citep{silander2012,xiang2013lasso} are exact methods that only work with a small number of nodes.

\section{Notation and Background}
\label{sec:preliminaries}

In this section, we establish the notation and provide context for the class of problems \eqref{eq:general_opt}.

\subsection{Nonlinear DAG models}
\label{sec:nonlinear_dag_models}

We let $\gG = (V,E)$ denote a \emph{directed} graph of $d$ nodes, with vertex set $V = [d] := \{1,\ldots, d\}$ and edge set $E \subset V \times V$, where $(i,j) \in E$ indicates the presence of a directed edge from node $i$ to node $j$.
For a graph $\gG$, we associate each node $i \in V$ to a random variable $X_i$, and use $X = (X_1,\ldots,X_d)$ to denote the $d$-dimensional random vector.

We consider \emph{structural equation models} \citep[SEMs][]{peters2017elements}, in which each node $X_j$ is determined by a function $f_j:\R^d\to\R$ of its parents and independent noise $z = (z_1, \ldots, z_d)\in\R^{d}$ as follows:
\begin{align}
    \label{eq:sem}
    X_j = f_j(X,z_j), \quad
    \partial_k f_j = 0 \text{ if } k\notin\PA_j^\gG,
\end{align}
 where $\PA_j^\gG = \{ i \in V \mid  (i,j) \in E\}$ denotes the set of parents of node $j$ in $\gG$.
Note that we write $f_j$ as a function of every other variable, and separately impose a restriction on the dependence through the partial derivatives, as in \citet{zheng2020}. This is equivalent to the usual formulation $X_j = f_j(\PA_j^\gG,z_j)$, and is adopted for mathematical convenience in the sequel.  Standard examples of SEMs include linear SEMs \citep[e.g.,][]{peters2014identifiability,loh2014high} and additive noise models \citep{peters2014causal}.

With this notation, the graphical structure implied by an SCM $f=(f_1,\ldots,f_d)$ can be represented by the following $d\times d$ weighted adjacency matrix:
\begin{align}
\label{eq:adj:scm}
    W = W(f)
    = (w_{ij}),
    \quad
    w_{ij} = \NormII{\partial_i f_j}.
\end{align}
In practice, a family of functions is defined to approximate the nonlinear functions $f_j$; 
common examples include multilayer perceptrons (MLP) \citep{zheng2020,lachapelle2019gradient}, and basis expansions \citep{zheng2020,buhlmann2014cam}. 
See Appendix~\ref{app:sec:more_details_F} for a detailed discussion on these families of functions.

We use $\Theta$ to denote all the model parameters used for approximating $f$.
However, not all of these parameters are utilized for inducing the graphical structure implied by $f$. To differentiate, we use $\theta \subset \Theta$ to denote the subset of parameters that are used for inducing the weighted adjacency matrix $W$, and $\ttheta = \Theta \setminus \theta$ to denote the remaining model parameters.
In other words, we have the following relationship: $W(f) = W(\Theta) = W(\theta)$.

To simplify notation and improve the clarity of presentation, we present the case where there is a single parameter $\theta_{ij}$ \footnote{$\theta_{ij}$ can be a vector, it is required that $[W(\theta)]_{ij}=0$ if and only if $\theta_{ij} =0$, see Appendix~\ref{app:sec:more_details_F} for more discussions. } per candidate edge $(i,j)$, i.e., $[W(\theta)]_{ij} = \theta_{ij}$ and $W(\theta)=\theta$.
However, note that all of our results hold for the general case and are thoroughly treated in the technical proofs provided in Appendix~\ref{sec:proofs}.

\subsection{Score functions}
\label{sec:score_functions}

The class of programs \eqref{eq:general_opt} requires a loss/score function $Q$. 
We briefly review commonly used scores in the literature.
Let $\rmX = [\mathbf{x}_1,\cdots,\mathbf{x}_d]\in \sR^{n \times d}$ denote the observed data matrix.
Let $\Theta_i$ denote the parameters used to approximate $f_i$,
we use $f_{\Theta_i}$ to denote $f_i$ approximated by $\Theta_i$.

Since the score function depends on the observed data, in this subsection, we use $Q(\Theta; \rmX)$ to denote the score on $\Theta$ given $\rmX$.
Then, some possible score functions include:

\textbf{Least squares.} $Q(\Theta; \rmX) = \frac{1}{2n}\sum_{i=1}^d\NormII{\mathbf{x}_i - f_{\Theta_i}(\rmX)}^2$ for linear SEMs with equal noise variances \citep{loh2014high}.

\textbf{Negative log-likelihood.} $Q(\Theta;\rmX) =\frac{1}{2}\sum_{i=1}^d \log(\NormII{\mathbf{x}_i - f_{\Theta_i}(\rmX)}^2) $ for additive SEMs with Gaussian errors \citep{buhlmann2014cam}.

\textbf{Logistic loss.} Let $\mathbf{1}_n$ denote the $n$-dimensional vector of ones. Then, we have $Q(\Theta;\rmX)  = \frac{1}{n}\sum_{i=1}^d \mathbf{1}_n^\top (\log (\mathbf{1}_n+\text{exp}(f_i(\rmX)))-\mathbf{x}_i\circ f_{\Theta_i}(\rmX))$ for generalized linear models with binary variables \citep{zheng2020}.

In the sequel, we simplify notation by writing $Q(\Theta)$ instead of $Q(\Theta;\rmX)$.

\begin{remark}
It is important to emphasize that in practical applications, the choice of score $Q$ is crucial: In order for solutions to this problem to be useful, ideally the minimizer(s) of $Q$ should correspond to the true underlying DAG. This problem has been extensively studied  \citep{geiger2002,chickering2002,van2013ell_, loh2014high,nandy2018, aragam2019globally}, so we do not pursue it further here. For example, in recent work, \citet{reisach2021beware} show how certain scores are not scale invariant, which may be an issue in practice, but is simply an artifact of the score function, as originally pointed out by \citet{loh2014high}. By contrast, our explicit goal is to study the optimization-theoretic aspects of objectives \eqref{eq:general_opt}, and not to propose new algorithms for learning causal DAGs.
\end{remark}

\subsection{Continuous non-convex characterizations of DAGs}
\label{sec:gradient_based_review}

To conclude this section, we next provide a brief overview of the existing options for the function $h$.
We remind the reader that for presentation simplicity we have $W(\theta) = \theta$, as discussed at the end of Section \ref{sec:nonlinear_dag_models}.
\begin{condition}\label{cond:h}
The function $h$ has the following form:
\[h(B) = \sum_{i=1}^dc_i\Tr( B^i),\]
where $c_i>0$ for any $i$.
\end{condition}
\begin{corollary}[\citealp{dennis2020} Theorem 1]\label{cor:h=0}
    If $h$ satisfies Condition \ref{cond:h}, then we have that $h(B)=0$ if and only if $B$ corresponds to a DAG, for any nonnegative matrix $B$.
\end{corollary}
By now the literature contains many different proposals of functions $h$ that satisfy Condition \ref{cond:h}; in this paper, we mostly focus on the following three:

\begin{enumerate}
    \item \textbf{The \notears formulation.} \citet{xun2018} were the first to propose a differentiable characterization of DAGs given by $h(B) = \Tr(e^{B}) - d$ for a nonnegative  matrix $B$.    
    \item \textbf{A polynomial formulation.} \citet{yu2019dag} proposed the use of $h(B) = \Tr((I+\nicefrac{1}{d}\ B)^d) - d$ for a nonnegative matrix $B$. \item \textbf{The \dagma formulation.} \citet{bello2022} proposed the use of $h(B) = -\log\det(I-B)$ for a nonnegative  matrix $B$ with spectral radius less than one. 
\end{enumerate}

Note that $B$ above is commonly defined as $B = \theta \circ \theta$, where $\circ$ denotes the Hadamard product.
In that case, it has been shown that $\nabla_{\theta} h(\theta\circ \theta) = 0$ if and only if $\theta$ is a DAG \citep[see][]{dennis2020}.
The latter implies that all stationary points of $h$ are global minima of $h$, a property known as invexity, as highlighted by \citet{bello2022}. 

\begin{remark}
    Our results are general and apply to any function  $h$ satisfying Condition \ref{cond:h}. 
    Thus, our results apply to any of the three $h$ functions mentioned above. 
\end{remark}

\subsection{Necessary and sufficient conditions for optimality}
\label{sec:nec_suff_cond}

\citet{dennis2020} first studied \eqref{eq:general_opt} from an optimality perspective.
The authors argued that the use of the Hadamard product $\theta\circ \theta$ leads to an undesirable property, namely, any feasible $\Theta$ in \eqref{eq:general_opt} cannot satisfy regularity conditions.
Motivated by this negative result, \citet{dennis2020} proposed an alternative, yet equivalent, formulation by replacing $h(\theta \circ \theta)$ by $h(\Abs{\theta})$. Reasoning similarly, we reformulate \eqref{eq:general_opt} as
\begin{equation}
\label{eq:abs_notears}
    \min_{\Theta} \; Q(\Theta) \quad \subjto \quad h(\Abs{\theta}) \leq 0.
\end{equation}
By writing $\theta = \theta^+ - \theta^-$, where $\theta^+= \max\{\theta,0\}$ and $\theta^-=\max\{-\theta,0\}$ denote the positive and negative parts of $\theta$, respectively.
Then, an equivalent smooth formulation is given by
\begin{align}
\label{eq:eq_sm_notear}
    \min_{\theta^+,\theta^-,\ttheta} &\; Q((\theta^+ - \theta^-,\ttheta)) \\
    \subjto &\; h(\theta^++\theta^-) = 0,\text{ and }\ \theta^+, \theta^- \geq 0. \notag
\end{align}
For clarity, we remind the reader that in \eqref{eq:eq_sm_notear} we have $\Theta = (\theta^+, \theta^-, \ttheta)$.
Then, the KKT conditions for \eqref{eq:eq_sm_notear} can be succinctly written as follows:
\begin{subequations}
\label{eq:kkt_con}
\begin{align}
     \frac{\partial Q(\Theta)}{\partial \theta_{ij}^{+}} + \lambda \frac{\partial h( \theta^++\theta^-)}{\partial\theta_{ij}^{+}} &= M_{ij}^{+} \geq 0, \label{eq:kkt_con a}\\
     -\frac{\partial Q(\Theta)}{\partial \theta_{ij}^{-}} + \lambda \frac{\partial h( \theta^++\theta^-)}{\partial\theta_{ij}^{-}} &= M_{ij}^{-} \geq 0, \label{eq:kkt_con aa}\\
    \theta_{ij}^{+}\circ M_{ij}^{+}= \theta_{ij}^{-}\circ M_{ij}^{-} &= 0, \label{eq:kkt_con b}\\
    \frac{\partial Q(\Theta)}{\partial \tilde{\theta}} &= 0 \label{eq:kkt_con c},
\end{align}
\end{subequations}
in addition to the feasibility conditions in \eqref{eq:eq_sm_notear}. where $M^{\pm}$ and $\lambda$ are the Lagrange multipliers of the constraints on $\theta^{\pm}$ and $h$, respectively. Here $\lambda \in \mathbb{R},M^{\pm}\geq 0$.

Briefly, \eqref{eq:kkt_con a}, \eqref{eq:kkt_con aa} and \eqref{eq:kkt_con c} results from dual feasibility and the stationarity condition, while \eqref{eq:kkt_con b} stems from complementary slackness.

The following useful theorem from \citet{dennis2020} establishes the connection between KKT satisfiability in \eqref{eq:eq_sm_notear} and local minimality in \eqref{eq:abs_notears} for \emph{linear SEMs} (i.e., $\ttheta = \emptyset$). 
\begin{theorem}[\citealp{dennis2020}, Theorem~7]
\label{theorem:kkt_local}
Assume that $Q$ is convex, $h$ satisfies the Condition \ref{cond:h}, and $\ttheta = \emptyset$.
If $(\theta^+,\theta^-)$ satisfies the KKT conditions in \eqref{eq:kkt_con}, 
then $\theta^+ - \theta^-$ is a local minimum of \eqref{eq:abs_notears}.
\end{theorem}
A key ingredient of our developments in the sequel is the following alternative characterization of the KKT conditions, which turns out to provide an algorithmically amenable first-order sufficient condition for \emph{local} optimality. We include a proof in Appendix~\ref{sec:proofs}.

\begin{lemma}
\label{lemma:characterization_kkt}
If $\Theta = (\theta^+,\theta^-,\ttheta)$ satisfies the following conditions:
\begin{enumerate}[(i)]
    \item For $\left\{(i,j)\mid[\nabla h(\theta^++\theta^-)]_{ij}>0\right\} \Rightarrow \theta_{ij}^{\pm}=0$.
    \item For $\left\{(i,j)\mid[\nabla h(\theta^++\theta^-)]_{ij}=0\right\} \Rightarrow  \frac{\partial Q(\Theta)}{\partial \theta_{ij}^{\pm}} = 0$.
    \item $\frac{\partial Q(\Theta)}{\partial \tilde{\theta}} = 0$.
    \item $\theta^+\geq 0, \theta^-\geq 0$.
\end{enumerate}
Then, we have that $\Theta$ is a KKT point of \eqref{eq:eq_sm_notear}.
Moreover, if the score $Q$ is convex, any such $\Theta = (\theta^+ - \theta^-,\ttheta)$ is also a local minimum for problem \eqref{eq:abs_notears}.
\end{lemma}

\begin{remark}
\label{rem:nonconvex}
If the score $Q$ is smooth but non-convex, then we can no longer use Lemma~\ref{lemma:characterization_kkt} to automatically promote KKT points to local minima.
Thus, in the sequel, whenever the score $Q$ is non-convex, all claims about local minima must be demoted to KKT points. 
\end{remark}

\section{Optimization Algorithm: Topological Swaps}
\label{sec:order}
Our key idea is to solve \eqref{eq:abs_notears} as a two-staged problem: in the inner stage, we solve \eqref{eq:abs_notears} to an additional constraint that makes the problem tractable, and in the outer stage, we search over the set of constraints. The critical innovation is in using our reformulation of the KKT conditions in guiding this search.
Our specific set of constraints relies on imposing an ordering over the variables. We briefly review such order constrained optimization below before formally introducing our overall approach.

\subsection{Background: Order-constrained optimization}
\label{sec:order_opt}

We leverage the following well-known observation: For a \emph{fixed} topological sort, problem \eqref{eq:abs_notears}, or equivalently \eqref{eq:eq_sm_notear}, can often be solved efficiently. 
We briefly review this material here for completeness.

Recall that a topological sort (or order) for $\gG$ is a partial ordering $\prec$ on the vertex set $V = [d]$ such that $X_i\to X_j\implies i\prec j$, here $X_i\rightarrow X_j$ means there exists an edge from $i$ to $j$. A directed graph is acyclic if and only if it has a topological sort, although this sort may not be unique.  Equivalently, we can view a topological sort as a permutation on $V$.
\begin{definition}[Topological sort]
A topological sort $\prec$ defines a permutation $\pi$ of the vertex set $V$ for $\gG$ by letting $\pi(j)$ be the $j$-th node in the ordering defined by $\prec$. In other words, if $X_{\pi(i)}\to X_{\pi(j)}$, then $i<j$.
\end{definition}
A similar definition carries over in the obvious way for weighted adjacency matrices $\theta$.
We furthermore call $G$ (resp. $\theta$) \emph{consistent} with $\pi$ if $\pi$ is a topological sort of $G$ (resp. $\theta$), and write this as $G\sim\pi$ (resp. $\theta\sim\pi$). 

Given a permutation $\pi$, we then have the following order-constrained optimization problem:
\begin{equation}
\label{eq:order_opt}
    \min_{\theta\sim\pi} \; Q(\Theta).
\end{equation}
Due to the order consistency constraint $\theta\sim\pi$, the acyclicity constraint $h(|\theta|)\le 0$ is automatically satisfied and hence can be omitted from \eqref{eq:order_opt}. 

We next reformulate \eqref{eq:order_opt} with explicit linear constraints. Moreover, in the sequel, we use $\Ordopt$ to denote any solution to this problem: 
\begin{align}\label{eq:order_opt_sol}
  \Theta_{\pi}^* = (\ordopt,\Tilde{\theta}_{\pi}^*) \in 
    \argmin_{\Theta} \; &Q(\Theta) \\
    \subto\; &\theta_{\pi(i),\pi(j)}=0,\; \forall j < i. \notag
\end{align}
\begin{remark}\label{remark:gd2stapoint}
    Our results only require solving \eqref{eq:order_opt_sol} up to stationarity.
    That is, we can first set $\theta_{\pi(i),\pi(j)} = 0$, for all $j<i$, and then use any off-the-shelf first-order optimizer \citep{boyd2004convex,nesterov2018lectures}  for the resulting (non)convex unconstrained problem.
\end{remark}

\subsection{Algorithm}

Motivated by the observations above, we propose a general bi-level algorithm based on finding the topological sort $\pi$ of an optimal scoring DAG.

For any $\Theta$ and $\tau,\xi>0$, define a set
\begin{align}
\label{eq:cand_swaps}
   \Y(\Theta,\tau,\xi) \defeq  \left\{ (i,j) \mid \left[\nabla h\left(|\theta|)\right)\right]_{ij}\leq \tau,
   \left\|\frac{\partial Q(\Theta)}{\partial \theta_{ij}}\right\|_1 > \xi \right\}. 
\end{align}
Given this machinery, the four main steps of our approach (Algorithm \ref{alg:pseudoalgo1}) are as follows:
\begin{enumerate}
    \item Initialize at an arbitrary sort $\pi$, and solve \eqref{eq:order_opt}.
    \item Define a candidate set of possible swaps by $\Y(\Ordopt,\tau_*,\xi^*)$ as defined in \eqref{eq:cand_swaps}, where $(\tau_*,\xi^*)$ are parameters chosen adaptively such that $|\Y(\Ordopt,\tau_*,\xi^*)|\approx \Ys$.
    \item Choose the best swap from this set to obtain a new topological sort; i.e., the swap that decreases the score $Q$ the most. 
    \item Repeat until there is no sufficient improvement in the score.
\end{enumerate}
There are several advantages to this approach:
\begin{itemize}
    \item Enforcing acyclicity is much simpler: Once a topological sort is fixed, acyclicity is automatically guaranteed and the optimization is straightforward and efficient (cf. Section~\ref{sec:order_opt}). Thus, there is no need to include $h(\Abs{\theta})$ directly in the optimization routines compared to \citet{xun2018}, which greatly simplifies implementation.
    \item We will only need to check (ii), (iii), and (iv) in Lemma \ref{lemma:characterization_kkt} in order to ensure the KKT conditions are satisfied, and computing the gradients $\nabla Q$, $\nabla h$ is easy. 
    Note that Condition (i) is to ensure $|\theta|$ is acyclic, which is always satisfied by the argument in the above item.
\end{itemize}
It is worth stressing that this is \emph{not the same} as greedily selecting individual edges as in GES \citep{chickering2002}: Each swap re-solves \eqref{eq:order_opt} \emph{globally}, and hence updates every edge.

Crucially, in the second step, it is not necessary to exhaustively check all possible swaps: By properly exploiting the KKT conditions as in Lemma~\ref{lemma:characterization_kkt}, we are able to limit the set of possible candidate swaps to $\Y(\Ordopt,\tau_*,\xi^*)$. This greatly improves the efficiency of the algorithm. Moreover, it is not necessary to find the swap that decreases the score the most in Algorithm \ref{alg:pseudoalgo1} line 9. Instead, any swap that decreases $Q$ could be used to accelerate our algorithm. This greedy strategy, which is explored in the appendices, can improve time efficiency while attaining comparable performances.

The main steps of our method are summarized in Algorithm~\ref{alg:pseudoalgo1}; a more comprehensive outline (for reproducibility purposes) can be found in the Appendix~\ref{app:alg} (Algorithm~\ref{alg:algo1}). 
The subroutine \textsc{FindParams} (detailed in Algorithm \ref{alg:update_parameter} in Appendix~\ref{app:alg}) aims to find appropriate values for $\tau$ and $\xi$ such that $|\mathcal{Y}(\Theta,\tau,\xi)| \approx s$.
In Algorithm \ref{alg:pseudoalgo1}, the notation $\Ys$ and $\Yl$ are used to denote small and large search spaces, respectively.

\begin{remark}
It is worth noting how the continuous formulation plays a critical role in Algorithm~\ref{alg:pseudoalgo1}: We use both the KKT conditions and the function $h$ in order to select candidate swaps (cf. \eqref{eq:cand_swaps}).
\end{remark}

\begin{algorithm}[!th]
\caption{\textsc{Topo}}
\label{alg:pseudoalgo1}
\begin{algorithmic}[1]
\REQUIRE Initial topological sort $\pi$, integers $\Ys$ and $\Yl$ with $\Yl>\Ys$, and score function $Q$.
\STATE \COMMENT{Here we use $\pi_{ij}$ to denote the new topological sort by swapping nodes $i$ and $j$ in $\pi$.}
\STATE $(\tau_*,\xi^*) \gets \textsc{FindParams}(\ordopt,\Ys)$ 
\STATE $\gS \gets \Y(\Ordopt,\tau_*,\xi^*)$
\WHILE{$\gS \neq \emptyset$}
\IF{$\exists (i,j) \in \gS$ s.t. $Q(\Theta^*_{\pi_{ij}}) < Q(\Ordopt)$} 
\STATE Update $\pi$ to be $\pi_{ij}$ that (most) decreases $Q$.
\STATE $\gS \gets \Y(\Ordopt,\tau_*,\xi^*)$ 
\ELSE 
\STATE $(\tau^*,\xi_*) \gets \textsc{FindParams}(\ordopt,\Yl)$ 
\STATE $\gS \gets \Y(\Ordopt,\tau^*,\xi_*)$ \hfill \COMMENT{Try a larger search space}
\IF{$\exists (i,j) \in \gS$ s.t. $Q(\Theta^*_{\pi_{ij}}) < Q(\Ordopt)$}
\STATE Update $\pi$ to be $\pi_{ij}$ that (most) decreases $Q$.
\STATE $\gS \gets \Y(\Ordopt,\tau_*,\xi^*)$  
\ELSE
\STATE $\gS \gets \emptyset$
\ENDIF
\ENDIF
\ENDWHILE
\ENSURE $\Ordopt$
\end{algorithmic}
\end{algorithm}

\subsection{Analysis}
\label{sec:analysis}

Intuitively, the idea behind Algorithm~\ref{alg:pseudoalgo1} is that it iteratively jumps between better and better local minimizers, until the candidate swaps given by \eqref{eq:cand_swaps} no longer offer any significant improvement in the score. This is achieved by exploiting the KKT conditions \eqref{eq:kkt_con}. 
In this section, we show that this is not just a heuristic: Under appropriate conditions, Algorithm~\ref{alg:pseudoalgo1} indeed decreases the score and always terminates at a local minimum or KKT point.

Before proving this, it is worth stressing why this is not obvious \emph{a priori}: Even if we solve \eqref{eq:order_opt} to global optimality (i.e., given the order constraint $\theta\sim\pi$), a global solution to \eqref{eq:order_opt} need not be a \emph{local} solution to \eqref{eq:abs_notears}. This stems from the fact that a DAG can have more than one topological sort, and the solutions to \eqref{eq:order_opt} for each sort need not coincide.

We begin with two important lemmas.

\begin{lemma}
\label{lemma:zero_entry}
If $(i,j)\in \mathcal{Y}(\Ordopt,0,0)$, then $\left(\ordopt\right)_{ij}=0$.
\end{lemma}

\begin{lemma}
\label{lemma:always_decrease}
If the score $Q$ is separable w.r.t $\theta$, i.e. $Q(\Theta) = \sum_{j} Q_j(\theta_j,\tilde{\theta})$ and $\gY(\Ordopt,0,0) \not = \emptyset$ for some topological sort $\pi$, then 
$$Q(\Theta^*_{\pi_{ij}})< Q(\Ordopt),$$ 
for every $(i,j)\in \Y(\Ordopt,0,0)$.
\end{lemma}

Lemma~\ref{lemma:always_decrease} has an important takeaway message: As long as we can find a pair of nodes $(i,j)\in\Y(\Ordopt,0,0)$---i.e. $\Y(\Ordopt,0,0)\ne\emptyset$---then we can find another topological sort with strictly smaller score. The difficult case is when $\Y(\Ordopt,0,0)=\emptyset$: What Algorithm~\ref{alg:pseudoalgo1} does is increase the thresholds $(\tau_*,\xi^*)$ just enough to make $\Y(\Ordopt,\tau_*,\xi^*)\ne\emptyset$.
Indicated by the previous observation, this suggests that placing node $i$ before node $j$ is likely (but not guaranteed) to decrease the score. There are many strategies for updating the topological sort to make this happen, but we adopt the simplest way, i.e., swapping the node $i$ and node $j$.

This previous discussion can be made more concrete via the following observation:
\begin{corollary}
\label{cor:emptyKKT}
If $\Y(\Ordopt, 0, 0)=\emptyset$, then $\Ordopt$ satisfies the KKT conditions in \eqref{eq:kkt_con}.
\end{corollary}

The following definition relates to the score $Q$ and is a relevant property for Theorem \ref{thm:score_decrease}.
\begin{definition}[Connected estimator]
Given a topological sort $\pi$, the estimator $\Ordopt$ is called connected if for any $i<j$ there is a directed path from node $\pi(i)$ to node $\pi(j)$ in $\ordopt$. 
\end{definition}
Equivalently, for any $i < j$, a connected estimator satisfies $\left[\nabla h(|\ordopt|)\right]_{\pi(j),\pi(i)}>0$. 
In general, we expect an estimator to be connected when sparse regularization is not used. It is worth noting that NOTEARS \citep{xun2018} \emph{without} explicit $\ell_{1}$ regularization is observed to return a connected estimator.

\begin{theorem}
\label{thm:score_decrease}
For any $h$ satisfies the Condition \ref{cond:h}. If the score $Q$ is convex {(resp. non-convex)} and $\Ordopt$ is connected for all $\pi$. 
Then Algorithm \ref{alg:pseudoalgo1} returns a local minimum ({resp. KKT point}) of problem \eqref{eq:abs_notears}, where the score is decreased at each iteration.
Moreover, the solution at each iteration is also a local minimum (resp. KKT point).
\end{theorem}

\begin{remark}
Although the proof of Theorem~\ref{thm:score_decrease} is deceptively simple, we stress that it is not \emph{a priori} obvious that swapping pairs of nodes will always decrease the score: Done na\"ively, this could increase the score. Our careful use of the KKT conditions precludes this behavior.
\end{remark}
The connected estimator assumption in Theorem \ref{thm:score_decrease} can be dropped whenever the score $Q$ is separable (e.g., least squares).
\begin{theorem}
\label{theorem:F_sep_score_decrease}
For any $h$ satisfies the Condition \ref{cond:h}. Assume that the score $Q$ is separable w.r.t $\theta$, i.e., $Q(\Theta) = \sum_{j} Q_j(\theta_j,\tilde{\theta})$.
If the score $Q$ is convex {(resp. non-convex)}, then Algorithm \ref{alg:pseudoalgo1} returns a local minimum ({resp. KKT point}) of problem \eqref{eq:abs_notears}, where the score is decreased at each iteration.
\end{theorem}

\subsection{Comparison to previous work}
\citet{dennis2020} first unveiled the connections between the KKT conditions in \eqref{eq:kkt_con} and local minimality in \eqref{eq:abs_notears} by studying a related problem with \emph{explicit edge absence constraints $\gZ$}. 
As such, it is instructive to compare these two approaches since there are some important distinctions. 
A first clear difference is that the KKTS algorithm by \citet{dennis2020} relies on an assumption they call \emph{irreducibility}, to ensure local minimality.

We provide a complete discussion on the irreducibility assumption of \citet{dennis2020} in Appendix~\ref{sec:prev_work} and focus on the main ideas here.
Briefly, KKTS \citep{dennis2020} uses a set of node pairs $\gZ$ to indicate which edges should be absent in the graph.
KKTS works by iteratively adding and removing elements to $\gZ$, and the algorithm stops once $\gZ$ is an irreducible set.
Then, \citet{dennis2020} show that when $\gZ$ is irreducible, the KKTS solution is a local minimum, provided additional assumptions such as the score being separable and convex.

In Appendix~\ref{sec:prev_work}, we show that irreducibility is not a necessary condition for optimality. 
We prove this by showing a simple example where an optimal solution can correspond to a \textit{reducible} set $\gZ$.
\begin{proposition}
\label{prop:kkts_sufficient}
    Irreducibility of the set $\gZ$ is sufficient but not necessary for KKTS to find a KKT point of problem \eqref{eq:eq_sm_notear}.
\end{proposition}
The above discussion should already mark a clear distinction of Algorithm \ref{alg:pseudoalgo1} to KKTS, i.e., our method \emph{does not} rely on the irreducibility assumption.
Finally, we note that the irreducibility assumption might seem a mild condition, however, it can have a severe effect on the runtime of KKTS as it will not stop until an irreducible set is found.

A second difference to KKTS is that our approach not only attempts to find an optimal solution but also attempts to find the local optimum with the lowest score possible.
This fact is a direct consequence of how Algorithm \ref{alg:pseudoalgo1} works, namely, at each iteration we look for a solution with lower score.
The fact that KKTS does not use the score $Q$ to guide their search procedure can result in solutions with high scores.
We next provide more details.
Full details can be found in Appendix~\ref{app:sec:3node_example}.
\begin{example}
\label{ex:3node}
Consider the following three-node linear SEM with standard Gaussian noise $z_j\sim \gN(0,1)$, for $i \in [3]$. 
Consider also that the score $Q$ is the population least square loss.
\begin{align}
\label{eq:counterexample}
    X_1 =  z_1, \quad 
    X_2 = a X_1 + z_2, \quad
    X_3 = b X_2 + z_3.
\end{align}
In  Appendix~\ref{app:sec:3node_example}, we show that for the linear model \eqref{eq:counterexample} there exists many values $a$ and $b$ where the solutions from KKTS and NOTEARS produce solutions with higher score w.r.t. Algorithm \ref{alg:pseudoalgo1}; moreover, NOTEARS produces non-optimal solutions.
This is illustrated in Appendix~\ref{app:sec:3node_example} for $a = 1, b = -0.55$. 
In each of these examples, our method can always return a solution that satisfies the optimality conditions in Lemma \eqref{lemma:characterization_kkt} and also attain the lowest score.
\end{example}

\section{Experiments}
\label{sec:exp}

\begin{table}[ht]
\centering
    \begin{tabular}{lcccc}
    \toprule
    Method & Metric 
    & $d=20$  &$d=40$ & $d=100$    
    \\
    \midrule
   
    \multirow{3}{*}{\textsc{Golem-EV}} 
    & KKT
    & 0&0&0
    \\
    
    &Loss 
    & $10.7\pm 0.12$ & $40.7\pm 4.8$ & $68.8\pm 3.9$
    \\
    & SHD  
    & $11.4\pm 3.4$ & $ 51.4\pm 28.3 $ & $145.2\pm 52.6$
    
    \\ \hline

    \multirow{3}{*}{\textsc{Notears}} 
    & KKT    & 0&0&0 \\
    & Loss   & $11.9\pm 0.1$ & $62.1 \pm 8.8$ & $73.1 \pm 7.6$ \\
    & SHD    & $28.6\pm 3.2$ & $ 129\pm 25.5 $ & $140.0 \pm 30.1$ \\ 
    \hline

    \multirow{3}{*}{\textsc{Nofears}} 
    & KKT & 1 & 1 & 1 \\
    & Loss  & $11.5\pm 0.3$ & $ 47.6\pm 1.6$ & $61.2 \pm 2.6$ \\
    & SHD  & $23.2\pm 4.5$ & $ 69.8\pm 16.0$ & $87.5 \pm 19.2$ \\ 
    \hline
    \multirow{3}{*}{\textsc{Notears-Topo}}      
    & KKT  & 1&1&1 \\
    &Loss  & $\Best{9.8\pm 0.1}$ & $\Best{38.4\pm 0.1}$ & $\Best{47.5\pm 0.1}$ \\
    & SHD  & $\Best{0.4\pm 0.2}$ & $9.2 \pm 0.8$ & $\Best{14.2 \pm 1.9}$\\ 
    \hline
    \multirow{3}{*}{\textsc{Random-Topo}} 
    & KKT  & 1&1&1 \\
    &Loss  & $\Best{9.8\pm 0.1}$ & $\Best{38.4 \pm 0.1}$ & $\Best{47.5\pm 0.1}$ \\
    & SHD  & $\Best{0.4\pm 0.2}$ & $\Best{8.6 \pm 0.9}$  & $16.3 \pm 2.6$ \\
    \bottomrule
    \end{tabular}
\caption{Experiments on linear DAGs with equal-variance Gaussian noise on ER4 graphs. 
The score is the least squares, and $d$ is the number of nodes. Our methods are \textsc{Random-Topo}, and \textsc{Notears-Topo}.}
\label{Table:linear_gauss_EV}
\end{table}

\begin{table}[ht]
\centering
	\begin{tabular}{lcccc}
	\toprule
	Method & Metric 
	& $d=20$  &$d=40$ & $d=100$ \\
	\midrule
	\multirow{3}{*}{\textsc{Golem-NV}} 
	& KKT	& 0&0&0\\
	&Loss   	& ${9.9\pm 0.6}$ & $15.2 \pm 1.3$ & $42.7\pm 3.5$\\
	& SHD 	& $\Best{2.3\pm 0.1}$ & $\Best{23.4 \pm 3.4}$  & $\Best{82.1 \pm 12.3}$ \\ 
	\hline
	\multirow{3}{*}{\textsc{Notears}} 
	& KKT	& 0&0&0\\
	&Loss   & ${13.8\pm 2.1}$ & $17.2 \pm 1.2$ & $50.6\pm 4.5$\\
	& SHD 	& ${7.3\pm 0.1}$ & $39.2 \pm 7.1$  & $138.1 \pm 23.6$	\\ 
	\hline
	\multirow{3}{*}{\textsc{Notears-Topo}} 
	& KKT & 1&1&1\\
	&Loss 	& $\Best{8.3\pm 1.2}$ & $\Best{13.2 \pm 2.1}$ & $\Best{35.1\pm 2.3}$\\
	& SHD 	& ${2.7\pm 3.2}$ & $26.3 \pm 4.2$  & $86.9 \pm 6.6$\\ 
	\hline
	\multirow{3}{*}{\textsc{Random-Topo}} 
	& KKT	& 1&1&1\\
	&Loss  	& ${8.9\pm 1.3}$ & $14.4 \pm 1.2$ & $39.2\pm 4.1$	\\
	& SHD	& ${3.3\pm 0.2}$ & $29.1 \pm 4.2$  & $106.4 \pm 11.6$	\\
	\bottomrule
	\end{tabular}
\caption{ Experiments on linear DAGs with unequal-variance Gaussian noise on ER4 graphs. 
The score is the log-likelihood with the minimax concave penalty (MCP) penalty, and $d$ is the number of nodes. Our methods are \textsc{Random-Topo}, and \textsc{Notears-Topo}.} 
\label{Table:linear_gauss_NV}
\end{table}

\begin{table}[ht]

\centering
    \begin{tabular}{lcccc}
    \toprule

    Method & Metric 
    & $d=10$  &$d=20$ & $d=50$  \\
    \midrule

    \multirow{3}{*}{\textsc{Golem-EV}} 

    & KKT 
     & 0 &0 &0 
    \\

    &Loss 
    & $4.3\pm 0.1$ & $4616.5 \pm 4163.2$ & $(7.4 \pm 7.4) \cdot 10^{18}$
    \\ 
    & SHD  
    & $6.5\pm 0.8$ & $85.1 \pm 6.4$ & $1152.5 \pm 2.6$
    \\
    \hline

    \multirow{3}{*}{\textsc{Notears}} 
     & KKT 
     & 0 &0 &0 
    \\

    & Loss  
    & $6.2\pm 0.2$ & $18.9 \pm 1.3$ & $(1.7 \pm 1.6) \cdot 10^{11}$
     \\
    
    & SHD 
    & $14 \pm 1.1$ & $79.5 \pm 2.1$ & $1198.1 \pm 5.3$
     \\

     \hline
    
    \multirow{2}{*}{\textsc{Notears-Topo}} 

    & KKT 
    & 1 &1 &1 
    \\

    &Loss   
    & $\Best{4.97\pm 0.1}$ &$\Best{9.92 \pm 0.1  }$ & $\Best{24.9 \pm 0.3}$
    \\
    
    &SHD  
    & $\Best{ 0.1\pm 0.1}$ & $\Best{0.7 \pm 0.2}$  & $\Best{19.4 \pm 5.5}$
    \\
     
    \hline

    \multirow{2}{*}{\textsc{Random-Topo}} 
    & KKT 
    & 1 &1 &1 
    \\
    
    &Loss   
    & $\Best{4.97\pm 0.1}$ & $10.3 \pm 0.2$ & $35.8 \pm 2.1$
    \\
    
    & SHD 
    & $\Best{0.1\pm 0.1}$ & $3.1 \pm 1.4$  & $155.6 \pm 17.9$
    \\

    \bottomrule
    \end{tabular}
\caption{Experiments on Fully connected linear DAGs with Gaussian noise. The score is least squares, and $d$ is the number of nodes. Our methods are \textsc{Random-Topo}, and \textsc{Notears-Topo}.}
\label{Table:fully_connected}
\end{table}

\begin{table}[ht]

\centering
    \begin{tabular}{lcccc}
    \toprule
    Method & Metric 
    & $d=10$  &$d=20$ & $d=40$    
    \\
    \midrule

    \multirow{3}{*}{\textsc{Notears-MLP}} 
    & KKT
    & 0&0&0
    \\
    
    &Loss   
    & ${7.2\pm 0.2}$ & $14.4 \pm 0.3$ & $28.5\pm 0.4$
    \\

    & SHD 
    & ${5.6\pm 0.7}$ & ${29.1 \pm 3.1}$  & ${112.3 \pm 20.2}$
    
    \\ \hline

    \multirow{3}{*}{\textsc{Notears-Topo}} 
    & KKT
    & 1&1&1
    \\
    
    &Loss   
    & $\Best{6.4\pm 0.1}$ & $ \Best{11.6\pm 0.1} $ & $\Best{22.8\pm 0.6}$
    \\

    & SHD 
    & $\Best{2.7\pm 0.5}$ & $\Best{12.1 }$  & $\Best{36.3 \pm 20.4}$
    
    \\ \hline

    \multirow{3}{*}{\textsc{True}} 
    & KKT
    & 1&1&1
    \\
    
    &Loss 
    & ${6.3\pm 0.1}$ & ${12.2 \pm 0.1}$ & ${23.4\pm 0.4}$
    \\
    & SHD  
    & ${2.1\pm 0.5}$ & $11.6 \pm 0.6$  & $36.1 \pm 2.2$
    
    \\ 

    \bottomrule
    \end{tabular}
\caption{  Experiments on Nonlinear Model with Neural Network on ER4 graphs. 
The score is least squares, and $d$ is the number of nodes. Our method is \textsc{Notears-Topo}. Here `True' means the solution of problem \eqref{eq:order_opt_sol} using the underlying true topological sort.} 
\label{Table:Nonlinear_model}
\end{table}

We compare our method against state-of-the-art solvers for \eqref{eq:general_opt}, namely, NOTEARS~\citep{xun2018,zheng2020}, NOFEARS (KKTS)~\citep{dennis2020},  and GOLEM~\citep{Ignavier2020}. 
For \textsc{Topo} (Algorithm \ref{alg:pseudoalgo1}), we consider the case of random initialization (denoted by starting with \textsc{`random'}), and initializing at the output of NOTEARS (denoted by starting with \textsc{`Notears'}).
Here, random initialization is conducted by sampling a topological sort $\pi$ uniformly at random, and solving problem \eqref{eq:order_opt} to get $\Ordopt$. Details for each experimental setting can be found in Appendix~\ref{app:sec:experiments}. Code is available at \href{https://github.com/Duntrain/TOPO}{https://github.com/duntrain/topo}.

Our main empirical results are shown in Tables~\ref{Table:linear_gauss_EV}, \ref{Table:linear_gauss_NV}, \ref{Table:fully_connected} and \ref{Table:Nonlinear_model}. 
In all the tables, we report: Whether or not the solution of the algorithms satisfies the KKT conditions (1 indicating that the method always returned a KKT point, and 0 indicating that it never returns a KKT point); the score/loss attained by the method; and the structural Hamming distance (SHD) w.r.t. the ground-truth DAG.

In Table~\ref{Table:linear_gauss_EV}, 
we observe that, as expected, NOFEARS and our algorithm are capable of returning a KKT point (and local minimum in this setting since $Q$ is convex).  
We also note that \textsc{Topo} with random initialization ($\topoRandom$) performs competitively in this case, even though the initial topological sort was randomly sampled. Moreover, notice that when initialized at the output of NOTEARS, our method ($\topoNotears$) improves the performance of NOTEARS \textbf{dramatically}. The latter demonstrates the usability of our method as a post-processing algorithm, as discussed in our contributions.

In Table~\ref{Table:linear_gauss_NV}, for a non-convex score, we observe that TOPO still obtains solutions satisfying KKT optimality and achieves the lowest scores.

In Table \ref{Table:fully_connected}, we study a very challenging setting where the underlying graph is a fully connected DAG. We observe that existing methods can perform reasonably well when the number of nodes is very small (e.g., 10) but their performance degrade severely for graph with larger number of nodes. 
In contrast, TOPO works remarkably well in this setting, which should come to no surprise since sparsity assumptions are not required, consistent with our analysis in Section \ref{sec:analysis}. 

In Table \ref{Table:Nonlinear_model},  we make explicit comparison to \textit{nonlinear} NOTEARS \citep{zheng2020}. Comparison against other methods is implicit in previous work \citep{zheng2020}. We observe that $\topoNotears$ outperforms all other methods and is close to the solution of problem \eqref{eq:order_opt_sol} using the true topological sort.

\subsection{Additional experiments}
In Appendix~\ref{app:sec:experiments}, we provide further experiments.
We consider linear models with different noise distributions (e.g., Gaussian, Gumbel and exponential) for \{ER$k$, SF$k$\} graphs, for $k \in \{1,2,4\}$. 
See Appendix~\ref{app:sec:linear_models}.
There, we observe that our methods even with random initialization still outperforms existing methods in terms of score and SHD, also the solutions are guaranteed to be local minimal. 
Additionally, our results are not specific to certain non-linearities.
To illustrate this, we run experiments on a logistic model (binary $X_j$), and neural networks.
See details in the Appendix~\ref{app:sec:nonlinear_models}.
Finally, we also report the runtime and scores of each method for linear and nonlinear models in Appendices \ref{app:sec:linear_models} and \ref{app:sec:nonlinear_models}.

We also analyze the sensitivity of the hyperparameters $\Ys$, and $\Yl$ on Algorithm \ref{alg:pseudoalgo1} (Appendix \ref{sec:sensitivity}), i.e., we study the effect of these hyperparameters on the cardinality of the search spaces (eq.\eqref{eq:cand_swaps}), and how many times our algorithm searches in a space of large cardinality. 
Moreover, we test our method using randomly chosen swapping set to demonstrate the effectiveness of \eqref{eq:cand_swaps} (Appendix~\ref{app:sec:random_swapping_sets}).
Finally, we include an analysis on structural accuracy vs iterations to track the performance of Algorithm \ref{alg:pseudoalgo1} (Appendix~\ref{app:sec:accuracy_vs_iters}).

\section{Conclusion}
\label{sec:conclusion}

Inspired by the KKT conditions, we developed new insights into the optimization-theoretic properties of NOTEARS objectives, and proposed a new bi-level algorithm with attractive local optimality guarantees. As a by-product, it can also improve the solutions of state-of-the-art solvers for \eqref{eq:general_opt} (e.g., NOTEARS, KKTS, GOLEM). 
Although proving convergence to a global minimizer is expected to be challenging, we have shown that our method has desirable properties for an optimization scheme: (a) It decreases the score in each iteration and (b) It is guaranteed to return a local minimizer (and hence also a KKT point). 
The key driver behind our approach is the interpretation of the KKT conditions as a proxy for choosing promising node swaps in a topological sort. An important open question for future work is the convergence of Algorithm~\ref{alg:pseudoalgo1}: What is its iteration and computational complexity?

It is also interesting to note that unlike previous methods that rely on explicitly enforcing acyclicity via $h(B)$, our approach only uses $h(B)$ indirectly in order to check the KKT conditions. This idea was already implicit in the KKTS method due to \citet{dennis2020}, and could lead to new insights into how to optimize NOTEARS objectives and other acyclicity-constrained problems.

\begin{ack}
 K. B. was supported by NSF under Grant \# 2127309 to the Computing Research Association for the CIFellows 2021 Project.
 B.A. was supported by NSF IIS-1956330, NIH R01GM140467, and the Robert H. Topel Faculty Research Fund at the University of Chicago Booth School of Business.
 This work was done in part while B.A. was visiting the Simons Institute for the Theory of Computing.
 P.R. was supported by ONR via N000141812861, and NSF via IIS-1909816, IIS-1955532, IIS-2211907.
 We are also grateful for the support of the University of Chicago Research Computing Center for assistance with the calculations carried out in this work.
 \end{ack}

\bibliography{main}
\bibliographystyle{agsm}
\clearpage
\appendix
\apptitle{Optimizing NOTEARS Objectives via Topological Swaps}

\section{Additional Discussion on Family of Approximators}
\label{app:sec:more_details_F}

Let $\mathcal{F} = \{f:f_j\in \mathcal{F}_j,\forall j\in [d]\}$ be a family of functions used to approximate the SCM in problem \eqref{eq:sem}.
In this section, we focus on the general case and discuss under what conditions that family $\mathcal{F}$ can be used to approximate $f_j$ and how our results apply in this general setting. 

We consider approximations $f = (f_1,\ldots,f_d)\in\mathcal{F}$ that are parameterized by $\theta$, i.e. $f(x):= f(x;\theta)$.
This defines $W(\theta):=W(f(\cdot;\theta))$ as the adjacency matrix defined by \eqref{eq:adj:scm}, which is characterized by $\theta$.  
Although the following definition is standard, we pause to make this precise since it is crucial in the development that follows:
\begin{definition}[sub-vector]
Given a vector $\beta = (\beta_1,\ldots,\beta_n)\in \sR^n$, and we say that $\alpha$ is a sub-vector of $\beta$ if and only if there is a subset $J = \{j_1,\ldots,j_k\}\subset \{1,2,\ldots,n\}$ such that $\alpha = \beta_{J} := (\beta_{j_1},\ldots,\beta_{j_k})$.
\end{definition}

Under the following general assumptions, our results and proof in Section \ref{sec:proofs} still apply without any modification:
\begin{enumerate}[(i)]
    \item The parametrization is separable in the following sense: $\theta = (\theta_1,\ldots,\theta_d)$ and each $f_j$ in \eqref{eq:sem} is only parameterized by the sub-vector $\theta_j$, i.e., $f_j(x;\theta) = f_j(x;\theta_j)$.
    \item There are sub-vectors $\theta_{ij}$ of $\theta_{j}$ that can reveal if there is no edge from node $i$ to node $j$ (i.e., $[W(\theta)]_{ij}=0$ if and only if $\theta_{ij}=0$.) In this case, the general definition in \eqref{eq:adj:scm} can be replaced with $[W(\theta)]_{ij}=\NormI{\theta_{ij}}$ without loss of generality.
\end{enumerate}

Write $\theta_{ij} = (\theta_{ijr})_r$ for each sub-vector $\theta_{ij}$. Since $[W(\theta)]_{ij}=\NormI{\theta_{ij}}$, we have
\[
	[W(\theta^++\theta^-)]_{ij} = \NormI{\theta_{ij}^++\theta^-_{ij}} 
	= \mathbf{1}^\top (\theta_{ij}^{+}+\theta^{-}_{ij}) 
	= \sum_{r}(\theta_{ijr}^{+}+\theta^{-}_{ijr})
\]
Therefore,
\[
	\frac{\partial [W(\theta^++\theta^-)]_{ij}}{\partial \theta_{ij}^\pm} = \left(\frac{\partial [W(\theta^+ +\theta^-)]_{ij}}{\partial \theta_{ijr}^\pm}\right)_r 
	= (1)_r 
	= \mathbf{1}.
\]

In Section \ref{sec:nec_suff_cond}, the KKT conditions for \eqref{eq:eq_sm_notear} involve the term $\frac{\partial h(W(\theta^++\theta^-))}{\partial \theta_{ij}^\pm}$. 
By the assumptions above, we see that $h(W(\theta^++\theta^-))$ is a function of $\theta_{ij}$ through $[W(\theta^++\theta^-)]_{ij}$, so that by the chain rule we have
\begin{align*}
    \frac{\partial h(W(\theta^++\theta^-))}{\partial \theta_{ij}^\pm} &=  \frac{\partial h(W(\theta^++\theta^-))}{\partial [W(\theta^++\theta^-)]_{ij}} \ \frac{\partial [W(\theta^++\theta^-)]_{ij}}{\partial \theta_{ij}^\pm} \\ 
    &=  [\nabla h(W(\theta^++\theta^-))]_{ij} \ \mathbf{1} \\ 
    &=  [\nabla h(W(|\theta|))]_{ij}\mathbf{1},
\end{align*}
this equality is crucial to Lemma~\ref{lemma:characterization_kkt}. 

We conclude by discussing three important special cases that satisfy the assumptions above: (1) Linear SEMs, (2) Multilayer perceptrons (MLPs), and (3) Basis expansions.

\paragraph{Linear SEMs.} 
	A linear SEM follows the following set of equations:
	\begin{align*}
	    X_j &= f_j(X, z_j)= {w}_j^\top X + z_j, \quad {w}_j\in\R^d, \quad \forall j\in [d],
	\end{align*}
	where $z_j \in \R$ represents the noise following any distribution. 
	Let $W = [w_1\mid w_2 \mid \cdots\mid w_d]\in \sR^{d\times d}$. 
	In this case, all the model parameters are $\theta = W$.
	The parameters related to node $j$ are $\theta_j = w_j$, thus, each function  $f_j$ is only characterized by $\theta_j$. 
	Thus, condition (i) above is clearly satisfied. 
	Furthermore, we have $\theta_{ij} = W_{ij}$ (the $(i,j)$-th entry of $W$), where clearly there is no edge from node $i$ to node $j$ if and only if $W_{ij}=0$. 
	Therefore, condition (ii) above is also satisfied.

\paragraph{Multilayer perceptrons (MLPs).}
	Let a multilayer perceptron (MLP) with $h$ hidden layers and a single activation $\sigma:\sR\rightarrow \sR$ be given by:
	\begin{align*}
	    \MLP(X;A^{(1)},\ldots,A^{(h)}) = \sigma(A^{(h)}\sigma(\ldots A^{(2)}\sigma(A^{(1)}x))),\\
	    A^{(\ell)}\in \sR^{m_{\ell}\times m_{\ell-1}},\qquad m_0 =d,\qquad m_h = 1.
	\end{align*}
	Then the nonlinear SCM with additive noise can be written as:
	\[
		X_j = f_j(X,z_j)= \MLP(X; A_j^{(1)},\ldots,A_j^{(h)})+z_j,
	\]
	where $z_j\sim \gN(0,1)$. 
	Let $\theta_j = (A_j^{(1)},\ldots,A_j^{(h)})$ denote the parameters for the $j$-th MLP, and let $\theta = (\theta_1,\ldots,\theta_d)$ denote all model parameters. 
	Define $\theta_{ij}$ to be the $i$-th column of $A_j^{(1)}$. 
	Since $\MLP(X;A_j^{(1)},\ldots,A_j^{(h)})$ is independent of $X_i$ if and only if $\theta_{ij}=0$ \citep[e.g.,][Proposition~1]{zheng2020}, we can define $[W(\theta)]_{ij} = \NormI{\theta_{ij}}$. 
	Then, in this case it is easy to check that conditions (i) and (ii) above are satisfied.
	
\paragraph{Basis expansion.}
	As an alternative to neural networks, we also consider the use of orthogonal basis expansions, as in \citep{zheng2020}. 
	Let $\{\varphi_r\}^\infty_r$ be an orthonormal basis of functions such that $\E[\varphi_r(X)] = 0$ for each $r$ and
	\[
		f(x) = \sum_{r=1}^\infty \alpha_r \varphi_r(x),\qquad \alpha_r = \int_{\sR^d}\varphi_r(x)f(x)dx.
	\]
	Consider additive models and one-dimensional expansions as follows:
	\[X_j = f_j(X,z_j) = \sum_{i\ne j}f_{ij}(X_i)+z_j =\sum_{i\ne j}\sum_{r=1}^\infty \alpha_{ijr}\varphi_r(x_i)+z_j. \]
	In this case, we let $\theta = (\alpha_{ijr})_{i,j,r}$ denote all model parameters, $\theta_j =(\alpha_{ijr})_{ir}$ denote all parameters related to node $j$, and $\theta_{ij} = (\alpha_{ijr})_{r}$ denote the parameters that model the absence of an edge from node $i$ to node $j$. 
	Additionally, set $[W(\theta)]_{ij} = \NormI{\theta_{ij}} = \sum_{r}\Abs{\alpha_{ijr}}$. 
	Similarly, it is easy to check that conditions (i) and (ii) above are both satisfied.

\section{Algorithm and Model Details}
\label{app:alg}

\subsection{Full Algorithm Description}

A full and reproducible outline of Algorithm~\ref{alg:pseudoalgo1} can be found in Algorithm~\ref{alg:algo1}.
Note that Algorithms \ref{alg:update_topo} (\textsc{UpdateSort}) and \ref{alg:update_parameter} (\textsc{FindParams}) are subroutines used by Algorithm \ref{alg:algo1}.

\begin{algorithm}[!ht]
\caption{\textsc{Topo}}
\label{alg:algo1}
\begin{algorithmic}[1]
\REQUIRE Given a topological sort $\pi$, two predefined numbers of swapping pairs $\Ys,\Yl$, number of search in large space $s_0$ and initialize corresponding $\Z_\pi$. Solve \eqref{eq:order_opt} to get $\Ordopt$, set $k\leftarrow0$ and $count\leftarrow0$.
\STATE $(\tau_*,\xi^*)\leftarrow\textsc{FindParams}(\Ordopt,\Ys)$ 
\STATE $(\tau^{*},\xi_*)\leftarrow\textsc{FindParams}(\Ordopt,\Yl)$.
\WHILE{  $\Y(\Ordopt,\tau_{*},\xi^{*}) \ne \emptyset$ }
\IF{ $\Y(\Ordopt,0,0) \ne \emptyset$}\label{alg:if1}
\STATE $\Y \gets \mathcal{Y}(\Ordopt,0,0)$
\FOR{$(i,j)\in \mathcal{Y}$}
\STATE  $\pi_{ij} \leftarrow \textsc{UpdateSort}(\ordopt,(i,j),opt = 2)$. 
\STATE Solve \eqref{eq:order_opt} to obtain $\Theta^*_{\pi_{ij}}$.
\ENDFOR
\ELSE \label{alg:else1}
\STATE $\Y \leftarrow \mathcal{Y}(\Ordopt,\tau_*,\xi^*)$
\FOR{$(i,j)\in \mathcal{Y}$ }
\STATE $\pi_{ij} \leftarrow \textsc{UpdateSort}(\ordopt,(i,j),opt = 1)$.
\STATE Solve \eqref{eq:order_opt} to obtain $\Theta^*_{\pi_{ij}}$.
\ENDFOR
\ENDIF \label{alg:endif1}
\IF{$\min_{(i,j)\in \mathcal{Y}}Q(\Theta^*_{\pi_{ij}})\leq Q(\Ordopt)$}
\STATE update $\pi := \arg\min_{\pi_{ij}} Q(\Theta^*_{\pi_{ij}})$ 
\ELSE
\IF {$k < s_0$}
\STATE  $\Y \leftarrow \mathcal{Y}(\Ordopt,\tau^*,\xi_*) \ne \emptyset$
\FOR{$(i,j)\in \mathcal{Y}$ }
\STATE$\pi_{ij} \leftarrow \textsc{UpdateSort}(\ordopt,(i,j),opt = 1)$.
\STATE Solve \eqref{eq:order_opt} to obtain $\Theta^*_{\pi_{ij}}$.
\ENDFOR
\IF{$\min_{(i,j)\in \mathcal{Y}}Q(\Theta^*_{\pi_{ij}})\leq Q(\Ordopt)$}
\STATE update $\pi := \arg\min_{\pi_{ij}} Q(\Theta^*_{\pi_{ij}})$ , $k\leftarrow k+1$
\ELSE
\STATE \textbf{Return} $\Ordopt$ \textbf{ then break}
\ENDIF
\ELSE
\STATE \textbf{Return} $\Ordopt$ \textbf{ then break}
\ENDIF
\ENDIF
\STATE Solve \eqref{eq:order_opt} to obtain $\Ordopt$
\STATE Update $\Y(\Ordopt,\tau_{*},\xi^{*})$
\STATE $count\leftarrow count+1$
\STATE $(\tau_*,\xi^*)\leftarrow\textsc{FindParams}(\Ordopt,\Ys)$ 
\label{alg:update_para1}
\STATE $(\tau^{*},\xi_*)\leftarrow\textsc{FindParams}(\Ordopt,\Yl)$.
\ENDWHILE 
\end{algorithmic}
\end{algorithm}

\begin{algorithm}[!ht]
\caption{\textsc{FindParams}}
\label{alg:update_parameter}
\begin{algorithmic}[1]
\REQUIRE Parameter $\Theta$, integer $q$ that controls the size of the search space.
\STATE Create a predefined  $T=\{\tau_1,\ldots,\tau_m\}$ and $\Xi=\{\xi_1,\ldots,\xi_l\}$.
\ENSURE $(\tau,\xi) = \argmin_{\tau \in T,\xi \in \Xi}\,\, \AAbs{q-\Abs{\Y(\Theta,\tau,\xi)}}$
\end{algorithmic}
\end{algorithm}

\begin{algorithm}[!ht]
\caption{\textsc{UpdateSort}}
\label{alg:update_topo}
\begin{algorithmic}[1]
\REQUIRE Parameter $\theta$ or topological sort $\pi$, $(i,j)$, $opt$. Initialize predefined $\epsilon\leftarrow10^{-8}$ (small).
\IF{$opt = 1$}
\STATE Swap nodes $i$ and $j$, and denote the new topological sort by $\pi_{ij}$
\ELSE
\STATE $\theta' \gets \theta$ 
\STATE $\theta'_{ij} \gets \theta_{ij} - \epsilon \frac{\partial Q(\Theta)}{\partial \theta_{ij}}$
\STATE  Find the topological sort of $W(|\theta'|)$, denoted as $\pi_{ij}$.
\ENDIF
\ENSURE $\pi_{ij}$
\end{algorithmic}
\end{algorithm}

\subsection{Additional Details on Hyperparameters}
In this section, we describe more details of the proposed order-based search method in Algorithm \ref{alg:algo1}. 
This involves initializing the number of swapping pairs $\Ys$ to define a small search space, the number of swapping pairs $\Yl$ to define a large search space,  and the maximum number of searches $s_0$ to perform in the large swapping-pairs space. 
From line 31 to 32, for each iteration, Algorithm \ref{alg:algo1} invokes Algorithm \ref{alg:update_parameter} to find out the best values for $\tau_*,\xi^*,\tau^*,\xi_*$ to control the number of swapping pairs in the search space.
For reference, the predefined $T$ and $\Xi$ that we used in Algorithm \ref{alg:update_parameter} are given in Table \ref{Table:hyperparams}. 

By tuning $\Ys$ and $\Yl$, we can control Algorithm \ref{alg:algo1} to quickly converge to a local minimum, or have the chance to escape to a better local minimum in case it finds a suboptimal one.  
Specifically, in Line 4, if the running solution $\ordopt$ does not satisfy the KKT conditions, then, as prescribed by Lemma \ref{lemma:always_decrease}, a better topological sort can be found.
In Line 10, although the running solution $\ordopt$ satisfies the KKT conditions, we use strict positive values for $\tau_*,$ and $\xi^*$ to expand the search space $\Y$ and consider potential swapping pairs that can lead us to a better local minimum.
Here, the parameter $s_0$ specifies how many times the algorithm can search in a large swapping-pairs space, with the goal to escape from a region of bad local minima. 
 These hyperparameters are determinant to Algorithm \ref{alg:algo1} performance and are tuned to balance the trade-off between accuracy and efficiency.

\begin{remark}
One merit of Algorithm \ref{alg:algo1} over prior work is that it does not only search for a local minimum but also tries to escape from bad local minima.
Thus, our algorithm usually attains the best scores among all DAG learning algorithms---when comparable.
A drawback of Algorithm \ref{alg:algo1} is the dependence on how large the search space is, which could be computationally intensive. 
Finally, Algorithm \ref{alg:algo1} solves (\ref{eq:edge_absence}) repeatedly, whose runtime also heavily determines the efficiency of Algorithm \ref{alg:algo1}.	
See Section \ref{app:sec:experiments} for runtime comparisons against existing methods.
\end{remark}

\begin{table}[H]
    \centering
    \begin{tabular}{cccccccc}
        \toprule
        Parameters & \multicolumn{7}{c}{Values}\\
        \cmidrule(lr){1-1}  \cmidrule(lr){2-8}\\
        $T$       & $0$  & $\num{1e-8}$   & $\num{1e-7}$ & $\num{1e-6}$ & $\num{1e-5}$&  $\num{1e-4}$ & $\num{1e-3}$\\
        \midrule
        \multirow{ 3}{*}{$\Xi$}   &  $0$     &  $\num{1e-7}$  & $\num{1e-6}$ & $\num{5e-6}$ & $\num{1e-5}$& $\num{5e-5}$ & $\num{1e-4}$\\
         &$\num{5e-4}$ &  $\num{1e-3}$     & $\num{5e-3}$ & $\num{1e-2}$ & $\num{5e-2}$ & $\num{1e-1}$ & $\num{5e-1}$\\
        & $1$ & $2$ & $5$ & $10$ & $15$& $20$& $40$ \\
        \bottomrule
    \end{tabular}
    \caption{Suggested values for parameters $T$ and $\Xi$ in Algorithm \ref{alg:update_parameter}.}
    \label{Table:hyperparams}
\end{table}

\section{Irreducibility and Comparison with Prior Work}

\subsection{Three-Node Example where KKTS and NOTEARS Fail}
\label{app:sec:3node_example}

In this section, we expand on Example \ref{ex:3node}. 
In particular, we show that our example was \textbf{not handpicked} but instead there exists several values $a$ and $b$ where the solutions from KKTS and NOTEARS either are DAGs with incorrect structure, or are non-optimal solutions, or both.
Recall that our example follows the following SEM:
\begin{align}
    X_1 &= z_1,\notag\\
    X_2 &= a X_1 + z_2, \label{eq:counterexample3}\\
    X_3 &= b X_2 + z_3, \notag
\end{align}
where $z_i \sim \gN(0,1)$ for $i\in [3]$.
For the purposes of this analysis we consider the class of SEMs such that $a^2 > b^2$. 
Then, the \textbf{true} adjacent matrix and topological sort are: 
\begin{equation*}
    W_{\text{true}} = \begin{pmatrix}
    0& a & 0\\
    0&0&b\\
    0&0&0
    \end{pmatrix}\qquad \pi_{\text{true}} = [1,2,3]
\end{equation*}

Letting \(X = \begin{pmatrix}
    X_1\\ X_2\\ X_3 
    \end{pmatrix},\) and  \(
    Z = \begin{pmatrix}
    z_1\\z_2\\z_3 
    \end{pmatrix}.\)
The SEM \eqref{eq:counterexample3} in vector form can be written as:
\[
	X = W_{\text{true}}^T X+Z.
\]
We will use the population least square (LS) as the score function, which is defined as follows:
\[Q(W) = \mathbb{E}\|X-XW\|_2^2 = \|(I-W_{\text{true}})^{-1}(I-W)\|_2^2\]
The motivation to choose such score is that it was shown by \citet{loh2014high} that $W_{\text{true}}$ is the unique global minimizer of the population LS for linear SEMs with equal noise variances. 
Next, we provide a closer look as to why our algorithm is capable of learning the correct structure, while KKTS and NOTEARS fail.

\subsubsection{The Output of KKTS}

	In the KKTS algorithm of \citet{dennis2020}, consider the set of edge absence constraints to be initialized at:
	$$\mathcal{Z}_0 = \{(1,1), (1,2), (1,3), (2,1), (2,2), (2,3), (3,1), (3,2), (3,3)\}.$$ 
	That is, the algorithm is initialized at the empty graph.
	Then, we have
	\[
	W^*(\mathcal{Z}_0)= 0, \qquad \nabla Q(W^*(\mathcal{Z}_0)) = -2\begin{bmatrix}
	    1& a& ab\\
	    a& a^2+1& b+a^2b\\
	    ab& b+a^2b& (ab)^2+b^2+1
	    \end{bmatrix},
	\]
	and $\{(i,j)\mid [\nabla h(|W^*(\mathcal{Z}_0)|)]_{ij} = 0 \} = \{(1,2),(1,3),(2,1),(2,3),(3,1),(3,2)\}$. 
	Recall that the KKTS algorithm will remove the pair $(i,j)$ from $\mathcal{Z}_0$ that satisfies the following property:
	\[
		(i,j) = \underset{\{(i,j)\mid [\nabla h(|W^*(\mathcal{Z}_0)|)]_{ij} = 0\}\cap \mathcal{Z}_0}{\arg\max} {\left|[\nabla Q(W^*(\mathcal{Z}_0))]_{ij}\right|}.
	\]
	Now, consider the case that $\max\{|a|,|ab|\}<|(a^2+1)b|$, then the pair $(3,2)$ is removed from $\mathcal{Z}_0$ and the resulting set of edge absence constraints is $\mathcal{Z}_1 = \{(1,1),(1,2),(1,3),(2,1),(2,2),(2,3),(3,1),(3,3)\}$.
	Then, at the next step we have:
	\[W^*(\mathcal{Z}_1) = \begin{bmatrix}
	    0& 0& 0\\
	    0& 0& 0\\
	    0& \frac{ a^2 b+b}{(ab)^2+b^2+1} & 0
	    \end{bmatrix},\qquad   \nabla Q(W^*(\mathcal{Z}_1)) = -2
	    \begin{bmatrix}
	    1&\frac{a}{(ab)^2+b^2+1}& ab\\
	    a&\frac{a^2+1}{(ab)^2+b^2+1}&a^2b+b\\
	    ab &0& (ab)^2+b^2+1
	    \end{bmatrix}, \]
	and $\{(i,j)\mid [\nabla h(|W^*(\mathcal{Z}_1)|)]_{ij} = 0\} = \{(1,2),(1,3),(2,1),(3,1),(3,2)\}$. 
	Consider now the case that $\max\{\frac{|a|}{(ab)^2+b^2+1},|ab|\}<|a|$, then the pair $(2,1)$ is removed from $\mathcal{Z}_1$ and the resulting set of edge absence constraints is $\mathcal{Z}_2 = \{(1,1),(1,2),(1,3),(2,2),(2,3),(3,1),(3,3)\}$.
	Then, at the next step we have:
	\[W^*(\mathcal{Z}_2) = \begin{bmatrix}
	    0& 0& 0\\
	    \frac{a}{a^2+1}& 0& 0\\
	    0& \frac{ a^2 b+b}{(ab)^2+b^2+1} & 0
	    \end{bmatrix},\qquad   \nabla Q(W^*(\mathcal{Z}_2)) = -2
	    \begin{bmatrix}
	    \frac{1}{a^2+1}&\frac{a}{(ab)^2+b^2+1}& ab\\
	    0&\frac{a^2+1}{(ab)^2+b^2+1}&a^2b+b\\
	    0 &0& (ab)^2+b^2+1
	    \end{bmatrix}, \]
	and $\{(i,j)\mid [\nabla h(|W^*(\mathcal{Z}_2)|)]_{ij} = 0\} = \{(2,1),(3,1),(3,2)\}$. 
	
	Now we note that $\{(i,j)\mid [\nabla h(|W^*(\mathcal{Z}_2)|)]_{ij} = 0\}\cap \mathcal{Z}_2 = \{(3,1)\}$. 
	However, we have $[\nabla Q(W^*(\mathcal{Z}_2))]_{31} = 0$, that is, even if we remove the pair $(3,1)$ from $\mathcal{Z}_2$, the corresponding $\mathcal{Z}_3 = \{(1,1),(1,2),(1,3),(2,2),(2,3),(3,3)\}$ leads to 
	\[W^*(\mathcal{Z}_3)= W^*(\mathcal{Z}_2), \qquad   \nabla Q(W^*(\mathcal{Z}_3))=\nabla Q(W^*(\mathcal{Z}_2)).\]
	
	Combining all the considerations on $a$ and $b$, we can conclude that as long as the following is satisfied:
	\begin{align*}
	    |ab|<|a|<|(a^2+1)b|,
	\end{align*}
	KKTS will find a DAG with incorrect structure, in fact, a DAG where all edges are reversed.
 Finally, one can easily see that there are infinitely many $a$ and $b$ satisfying the above condition. For example, let $a = 1$ and $b = -0.55$. We want to emphasize here that our result is also consistent with the result returned by the Python program provided by \citet{dennis2020}.

\subsubsection{The Output of NOTEARS}

	Since the NOTEARS implementation by \citet{xun2018} uses the augmented Lagrangian method and solve each inner unconstrained subproblem using the Quasi-Newton L-BFGS method, it is impossible to derive an analytical solution. 
	Instead, we directly verify the DAG solution returned by NOTEARS by setting the ground-truth SEM $a = 1$, $b = -0.55$.
	\[
	W_{\text{notears}} = 
	\begin{bmatrix}
	0&$\num{1.49e-04}$ &$\num{-7.00e-07}$\\
	$0.16$& 0& $-1.55$\\
	$-0.22$&$\num{-1.59e-05}$&0
	\end{bmatrix}
	\]
	We can note that $W_{\text{notears}}$ is not `exactly' a DAG.
	Thus, we use a threshold to remove small entries in $W_{\text{notears}}$.
	The resulting adjacency matrix is now a DAG and is denoted by $W_{\text{notears\_thres}}$.
	One can clearly see that the NOTEARS solution does not recover the true structure.
	\[
	W_{\text{notears\_thres}} = 
	\begin{bmatrix}
	0&$0$ &$0$\\
	$0.16$& 0& $-1.55$\\
	$-0.22$&$0$&0
	\end{bmatrix}
	\]
	Let us now check the KKT conditions given in Lemma \ref{lemma:characterization_kkt}.
	We have $\{(i,j)\mid [\nabla h(|W_{\text{notears\_thres}}|)]_{ij} = 0\}=\{(3,1),(2,3),(2,1)\}$, and $[\nabla Q(W_{\text{notears\_thres}})]_{21}\ne 0$, $[\nabla Q(W_{\text{notears\_thres}})]_{23}\ne 0$ and $[\nabla Q(W_{\text{notears\_thres}})]_{31}\ne 0$.
	We can then observe that NOTEARS fails at both outputting the true DAG structure, and a DAG that is a local minimum.
	
\subsubsection{The Output of TOPO (Algorithm \ref{alg:algo1})}
	To study how TOPO works for this 3-node example,
	we first calculate the loss of all possible topological sorts in the following table.
	\begin{table}[H]
	\centering
	\begin{tabular}{cc}
		\toprule
		 Topological sort $\pi$ & Score  \\ 
		 \midrule
		    $(1,2,3)$  & $3$ \\ 
		    $(1,3,2)$  & $2+b^2+\frac{1}{1+b^2}$ \\ 
		    $(2,1,3)$  & $2+a^2+\frac{1}{a^2+1}$ \\ 
		    $(3,1,2)$  & $1+b^2+(ab)^2+\frac{1}{1+b^2}+\frac{1+b^2}{1+b^2+(ab)^2}$\\ 
		    $(2,3,1)$  &  $2+a^2+\frac{1}{a^2+1}$\\ 
		    $(3,2,1)$ &  $\frac{1}{1+a^2}+\frac{1+a^2}{1+(ab)^2+b^2}+1+b^2+(ab)^2$\\ 
	    \bottomrule
	\end{tabular}
	\end{table}
	
	Now, rewriting eq.\eqref{eq:order_opt_sol} for the case of linear models, we have:
	$$W_{\pi}^*\in\underset{W\sim \pi}{\arg\min} ~ Q(W).$$ 
	To understand why TOPO is capable of returning the correct true structure, we next define adjacent topological sorts.
	\begin{definition}
	For two topological sorts $\pi_1$ and $\pi_2$, we say that $\pi_1$ and $\pi_2$ are \emph{adjacent} if there exists a pair of nodes in $\pi_1$ such that when swapped the resulting topological ordering is $\pi_2$.
	\end{definition}
	Recall that $\pi_{\text{true}} = (1,2,3)$.
	The following statement explains precisely the success of Algorithm \ref{alg:algo1}.
	\begin{corollary}
	Assume that $a\ne 0$, $b\ne 0$, and $a^2>b^2$, then for any topological sort $\pi\ne \pi_{\text{true}}$, we have that $Q(W_{\pi}^*) > Q(W^*_{\pi_{\text{true}}}) = Q(W_{\text{true}}) = 3$.
	Moreover, there always exists an adjacent topological sort $\pi_{\text{adj}}$ such that $Q(W_{\pi}^*)>Q(W_{\pi_{\text{adj}}}^*)$. 
	In other words, for any initial topological sort $\pi\ne (1,2,3)$, TOPO (Algorithm \ref{alg:algo1}) can always return $\pi_{\text{true}}$ and $W_{\text{true}}$ at last.
	\end{corollary}
	All the situations are summarized in the Table \ref{Table:Three-nodes}.
	\begin{table}[H]
	\centering
	\begin{tabular}{||cc||ccc||c||}
	\hline
	Current order $\pi$ & Adjacent order $\pi_{\text{adj}}$ & Current loss& & Better loss \\ \hline
	    $(1,2,3)$  & $(1,2,3)$ & $3$ &$\geq$& $3$\\ \hline
	    $(1,3,2)$  & $(1,2,3)$ & $2+b^2+\frac{1}{1+b^2}$ &$\geq$& $3$ \\ \hline
	    $(2,1,3)$  & $(1,2,3)$ & $2+a^2+\frac{1}{a^2+1}$ &$\geq$& $3$\\ \hline
	    $(3,1,2)$  & $(1,3,2)$ & $1+b^2+(ab)^2+\frac{1}{1+b^2}+\frac{1+b^2}{1+b^2+(ab)^2}$ &$\geq$& $2+b^2+\frac{1}{1+b^2}$ \\ \hline
	    $(2,3,1)$ & $(1,3,2)$ & $2+a^2+\frac{1}{a^2+1}$ &$\geq$& $2+b^2+\frac{1}{1+ b^2}$ \\ \hline
	    $(3,2,1)$ & $(1,2,3)$ & $\frac{1}{1+a^2}+\frac{1+a^2}{1+(ab)^2+b^2}+1+b^2+(ab)^2$ &$\geq$& $3$ \\ \hline
	\end{tabular}
	\caption{There always exists an adjacent topological sort whose score is strictly less than the current topological sort.}
	\label{Table:Three-nodes}
	\end{table}

\subsubsection{Analysis}

In this section we aim to provide more intuition as to why KKTS and NOTEARS fail in the above example.
Regarding the KKTS algorithm, it removes a pair $(i,j)$ from the set of edge absence constraints at each iteration.
Loosely speaking, this is equivalent to adding an edge $X_i \to X_j$ at each iteration, which implies that node $i$ must appear before node $j$ in the topological sort, and such relative ordering in the topological sort \emph{will never be reversed} in later iterations of the algorithm. 
Therefore, once a wrong pair $(i,j)$ is removed from the set of edge absence constraints, the KKTS algorithm has no ability to correct such erroneous step.
Although KKTS ensures that a local minimum is returned, it can learn a completely erroneous DAG structure, as shown in our example above. 
Regarding NOTEARS, as indicated by \citet{dennis2020}, the algorithm does not guarantee to return a local minimum even under the right formulation, and in most cases the NOTEARS solution is neither the correct structure nor a local minimum.

In contrast to KKTS and NOTEARS, TOPO can return correct structure in the example above regardless of the initial topological sort.
Our swapping strategy allows TOPO to change the topological sort in each iteration; importantly, while TOPO checks the optimality conditions, it uses the score function as the only criterion to find another topological ordering with better score, thus, jumping from one local optimum to a better local optimum. 

\subsection{Detailed Discussion About Irreducibility}
\label{sec:prev_work}

In this section, we discuss the differences between our method and the KKTS algorithm \cite{dennis2020}. 
One obvious difference is the type of constraint used. 
Another important difference is that the KKTS algorithm proposed by \citet{dennis2020} relies on an assumption they call \emph{irreducibility}, which we will show by example is \textbf{not needed} in general.

In this section, we consider one of special case of \eqref{eq:sem}: Linear SEMs, which is studied by previous work \cite{xun2018,dennis2020}.
\begin{align}
\label{eq:linsem}
    X_j &= {w}_j^\top X + z_j,
    \quad {w}_j\in\R^d
\end{align}
Let $W = [w_1,\ldots,w_d]$. In this case, $\theta$ is equivalent to $W$ and $\theta_{ij}$ is equivalent to $W_{ij}$. Therefore, to be consistent with previous works \cite{xun2018,dennis2020}, we use $W$ to replace $\theta$.

To connect the KKT conditions and local minimiality, \citet{dennis2020} used a related problem with explicit edge absence constraints, which correspond to zero-value constraints on the matrix $W$.
Specifically, given a set $\Z \subset V \times V$, their explicit edge absence constrained problem is given by:
\begin{equation}
\label{eq:edge_absence}
    \min_{W} \; Q(W; \rmX)  \quad \subjto \quad W_{ij} = 0, \forall\; (i,j) \in \Z.
\end{equation}
Following the notation in \citet{dennis2020}, we denote its optimal solution by $W^*(\Z)$. As with \eqref{eq:order_opt}, this problem can be solved efficiently---in fact, \eqref{eq:order_opt} is just a special case of \eqref{eq:edge_absence} with $\Z=\Z_\pi$, where
\begin{align*}
    \Z_\pi
    :=\left\{(\pi(j),\pi(i)) \mid  i<j\right\}.
\end{align*}

Driven by Theorem \ref{theorem:kkt_local}, the KKT-informed local search (KKTS) algorithm in \citet{dennis2020} repeatedly solves the edge absence problem \eqref{eq:edge_absence} for different $\Z$.
The KKTS algorithm stops once an \emph{irreducible} set $\Z$ is found, and the output $W^*(\Z)$ is guaranteed to be  a local minimum for problem \eqref{eq:abs_notears}. 
\citet{dennis2020} define irreducibility of a set $\Z$ as follows:
\begin{definition}[Irreducibility, \citealp{dennis2020}]
\label{def:irreducibility}
A set $\Z$ is called irreducible if $$(i,j)\in\Z\Rightarrow\left(\nabla h(\Abs{W^*(\Z)})\right)_{ij} > 0.$$
\end{definition}
Although irreduciblity of $\Z$ is a sufficient condition for a feasible solution to be a KKT point \citep[Theorem~8,][]{dennis2020}, it is not necessary, as the following example shows.

\def\truedag{W^{\dagger}}
Fix indices $i_{0} < j_{0}$ and define a ground truth DAG $\truedag$ by
\begin{equation*}
    (\truedag)_{i,j} = \begin{cases}
        1 & \text{if } i=i_0,j = j_0, \\
        0 & \text{otherwise}.
    \end{cases}
\end{equation*}
Let $\Omega= [d]\times[d]$ denote all pairs of indices, we initialize $\Z_0 = \Omega \setminus \{(i_0,j_0)\}$. 
Recall that the KKTS algorithm repeatedly removes elements from $\Z$ until it is irreducible. 
Let us assume KKTS takes $m$ steps, the elements removed in order are $(i_1,j_1),(i_2,j_2),\ldots (i_m,j_m)$ and define $\Z_k = \Omega \setminus \{(i_0,j_0),(i_1,j_1)\ldots,(i_k,j_k)\}$.

\begin{lemma}
\label{lemma:example}
Assume $Q$ is separable and that $\truedag$ is the unique global minimum of \eqref{eq:abs_notears}. Then initializing KKTS with $\Z_0 = \Omega \backslash \{(i_0,j_0)\}$, we have the following:
\begin{enumerate}
    \item $W^*(\Z_k) = \truedag$ for each $k=0,1,\ldots m$,
    \item $\Z_0, \ldots, \Z_{m-1}$ are not irreducible,
    \item $\Z_m = \{(j_0,i_0)\}\union\{(i,i)| i=1,\ldots,d\}$ is irreducible.
\end{enumerate}
\end{lemma}

In other words, the global minimum $\truedag$ is a KKT point that is not always irreducible, although it can be written in terms of an irreducible $\Z$. It is easy to construct models \eqref{eq:linsem} and score functions $Q$ such that $\truedag$ is a global minimizer: Simply choose the population least squares score with $z_j\sim \mathcal{N}(0,1)$ for each $j$; see \citet{loh2014high} for details.

\begin{corollary}\label{cor:sufficient}
Irreduciblity of $\Z$ is sufficient but not necessary for KKTS to find a KKT point of problem \ref{eq:eq_sm_notear}.
\end{corollary}

Lemma \ref{lemma:example} has direct implications when the underlying DAG is sparse.
If the initial $\Z_0 = \Omega$, KKTS needs to remove most of the elements in $\Z_0$ to reach an irreducible $\Z$, thus, it can be computationally intense and inefficient. 
Moreover, the score function in KKTS has a sparse regularization, and it can return the wrong $W^*(\Z)$ even if the current $\Z$ characterizes the edge absences of the ground truth exactly.

\section{Proofs of Technical Results}
\label{sec:proofs}
In this section, we present the proofs of lemmas and theorems in detail. First, let us discuss more on how to solve problem \eqref{eq:order_opt_sol} in the algorithm \ref{alg:algo1} which is helpful for our proof. In problem \eqref{eq:order_opt_sol}, we can eliminate the constraint $\theta_{\pi(i),\pi(j)}=0,\forall i>j$ by plugging them back to the objective function, then it is equivalent to the following unconstrained optimization problem, Note that we can write $\Theta = (\theta_{\pi(i),\pi(j)},\theta_{\pi(m),\pi(n)},\Tilde{\theta})$, where $i>j$ and $n>m$. In this case, $\Theta = (0,\theta_{\pi(m),\pi(n)},\Tilde{\theta})$
\begin{align*}
    (\{\theta^*_{\pi(m),\pi(n)}\}_{n>m},\Tilde{\theta}^*) = \arg\min Q(\Theta)=\arg\min Q((0,\theta_{\pi(m),\pi(n)},\Tilde{\theta}))
\end{align*}
Therefore, we can use any off-shelf optimizer that can solve such unconstrained optimization to a stationary point that will be suitable for our purpose, i.e., gradient descent or Adam \citep{kingma2014adam}. Throughout the proof, we repeatedly use the following $\frac{\partial Q(\Theta)}{\partial \Tilde{\theta}^*} = 0$ and $\frac{\partial Q(\Theta)}{\partial \theta^*_{\pi(m),\pi(n)}} = 0, \forall~ n>m, \forall ~\pi$. At last, we can construct $\theta^+,\theta^-$ from $\theta$ by $\theta^+ = \max\{\theta,0\}$ and $\theta^- = \max\{-\theta,0\}$.

\subsection{The Extension of Theorem \eqref{theorem:kkt_local}}
The proof of Theorem~\ref{theorem:kkt_local} in \citep{dennis2020} is for the case where the adjacency matrix $W$ does not have any parametrization. 
For completeness and ease of reference, we state the generalization to general parametrizations here:
\begin{theorem}[Theorem \eqref{theorem:kkt_local} in our content]
\label{theorem:local_point}
Assume that $Q(\theta)$ is convex. Then if $(\theta^+,\theta^-,\Tilde{\theta})$ satisfies the KKT condition in \eqref{eq:kkt_con}, $(\theta^*,\Tilde{\theta})$ is a local minimum for problem \eqref{eq:abs_notears}, where $ \theta^* = \theta^+-\theta^-$.
\end{theorem}
Although the proof is similar, we include a proof for completeness in Appendix~\ref{app:proof:local_point}.

\subsection{Proof of Lemma \ref{lemma:characterization_kkt}}
\begin{proof}
Let us denote $\theta_{ij}$ as $(\theta_{ijr})_{r}$, here note $\theta_{ij}$ is a vector and its each component is denoted as $\theta_{ijr}$, where $r=1,\ldots$. Therefore,
\[\frac{\partial Q(\Theta)}{\partial \theta_{ij}^\pm} = \left(\frac{\partial Q(\Theta)}{\partial \theta_{ijr}^\pm}\right)_{r}\]

Let us simplify term $\frac{\partial h(W(\theta^++\theta^-))}{\partial \theta_{ij}^\pm}$. First, note that
$$[W(\theta^++\theta^-)]_{ij} = \NormI{\theta_{ij}^++\theta^-_{ij}} =\mathbf{1}^\top (\theta_{ij}+\theta^-_{ij}) = \sum_{r}(\theta_{ijr}+\theta^-_{ijr})$$ (here we use the fact $\theta_{ij}^+\geq0, \theta^-_{ij}\geq 0$). Remember  $h(W(\theta^++\theta^-))$ is a function of $\theta_{ij}$ through $[W(\theta^++\theta^-)]_{ij}$, we can use chain rule
\begin{align*}
    \frac{\partial h(W(\theta^++\theta^-))}{\partial \theta_{ij}^\pm} = & \frac{\partial h(W(\theta^++\theta^-))}{\partial [W(\theta^++\theta^-)]_{ij}}\ \frac{\partial [W(\theta^++\theta^-)]_{ij}}{\partial \theta_{ij}^\pm}
    \\ = & [\nabla h(W(\theta^++\theta^-))]_{ij}\ \mathbf{1}
    \\ = & [\nabla h(W(|\theta|))]_{ij}\ \mathbf{1}
\end{align*}
First, for any $(i,j)$ such that
\begin{align*}
     \left[\nabla h(W(\Abs{\theta}))\right]_{ij} = \left[\nabla h(W(\theta^++\theta^-))\right]_{ij} > 0,
\end{align*}
we set
\begin{equation*}
    \lambda > \max_{(i,j):[\nabla h(W(\Abs{\theta}))]_{ij}>0} \frac{\NormI{\partial Q(\Theta)/\partial \theta_{ij}^{\pm}}}{\left[\nabla h(W(\Abs{\theta}))\right]_{ij}}.
\end{equation*}
Therefore, \eqref{eq:kkt_con a} and \eqref{eq:kkt_con aa} are satisfied with $M^+_{ij} > 0$ and $M^-_{ij} > 0$. 
From condition (i), we have $\theta_{ij} = 0$, that is,  $\theta^\pm_{ij}=0$, thus, \eqref{eq:kkt_con b} is satisfied since $\theta^+_{ij}\circ M^+_{ij}=\theta^-_{ij}\circ M^-_{ij}=0$. 

Second, for any $(i,j)$ such that
\begin{align*}
     \left[\nabla h(W(\Abs{\theta}))\right]_{ij} = \left[\nabla h(W(\theta^++\theta^-))\right]_{ij} = 0,
\end{align*}
we have from \eqref{eq:kkt_con a} and \eqref{eq:kkt_con aa}
\begin{equation*}
     \frac{\partial Q(\Theta)}{\partial \theta_{ij}^+} = M_{ij}^+ \geq 0.\quad  -\frac{\partial Q(\Theta)}{\partial \theta_{ij}^-} = M_{ij}^- \geq 0.
\end{equation*}
It is also known that
\begin{align*}
    \frac{\partial Q(\Theta)}{\partial \theta_{ij}^+}=\frac{\partial Q(\Theta)}{\partial \theta_{ij}^-}
\end{align*}
From condition (ii), we set corresponding $M_{ij}^\pm=0$, then \eqref{eq:kkt_con a} is satisfied. 
We also have $\theta^\pm_{ij}\circ M^\pm_{ij} = 0$, hence \eqref{eq:kkt_con b} is satisfied. From (iii), \eqref{eq:kkt_con c} is satisfied. From (iv), we know $\theta^+\geq,\theta^-\geq 0$. Also, it is obvious that $\forall (i,j)$, we have $\theta\circ\nabla h(\theta^++\theta^-)  = 0$, it is equivalent to $h(\theta^++\theta^-) = 0$ (\citealp{dennis2020}, Lemma 4). The feasibility conditions in \eqref{eq:eq_sm_notear} are also satisfied. Thus, $(\theta^+, \theta^-)$ satisfies the KKT conditions in \eqref{eq:kkt_con}. 

Finally, from Theorem \ref{theorem:local_point}, $(\theta^+ - \theta^-,\Tilde{\theta})$ is a local minimum for problem \eqref{eq:abs_notears} if $Q(\Theta)$ is convex.
\end{proof}

\subsection{Proof of Lemma \ref{lemma:zero_entry}}
\begin{proof}
Assume $p<q$, node $\pi(p)$ comes before $\pi(q)$ in $\pi$ by the definition of topological sort, so there is no directed walk from $\pi(q)$ to $\pi(p)$, which implies $\left(\nabla  h(W(|\ordopt|))\right)_{\pi(p),\pi(q)}=0$ (\citealp{dennis2020}, Lemma~7) and $\left(W(|\ordopt|)\right)_{\pi(q),\pi(p)}=0$. 
By the optimality conditions of \eqref{eq:order_opt_sol}, $\frac{\partial Q(\Ordopt)}{\partial \theta_{\pi(p),\pi(q)}}=0$. In other word, possible elements in $\mathcal{Y}(\Ordopt,0,0)$ must has formula $(\pi(q),\pi(p))$ where $p<q$. Therefore, $(\ordopt)_{\pi(q),\pi(p)}=0$. By the definition of $[W(|\ordopt|)]_{\pi(q),\pi(p)} = \NormI{(\ordopt)_{\pi(q),\pi(p)}} = 0 $
\end{proof}

\subsection{Proof of Lemma \ref{lemma:always_decrease}}
\begin{proof}
Let any $(i,j)\in \Y(\Ordopt,0,0)$, then $\left(\nabla h (W(\Abs{\ordopt}))\right)_{ij}=0$ and $\frac{\partial Q(\Ordopt)}{\partial \theta_{ij}}\not = 0$, it indicates there is no directed walk from $\boldsymbol{j}$ to such $\boldsymbol{i}$. From Lemma \ref{lemma:zero_entry}, $\left(\ordopt\right)_{ij}=0$. Changing the value of $\left(\ordopt\right)_{ij}$ introduces new edge which can create a cycle, however, from Lemma 6 in \citet{dennis2020}, changing the value of $\left(\ordopt\right)_{ij}$ cannot create directed walks from $j$ to $i$, by the assumption of separability of $Q(\Theta)$ and following the same argument of proof of Lemma 8 in \citet{dennis2020}, changing the value of $\left(\ordopt\right)_{ij}$ will not create cycle. Therefore, $\left(\ordopt\right)_{ijr}$ can be increased or decreased ($\frac{\partial Q(\ordopt)}{\partial \theta_{ijr}}<0$ or $\frac{\partial Q(\ordopt)}{\partial \theta_{ijr}}>0$ ) to reduce the loss function while maintaining feasible, which implies $W(|\widetilde{\theta}|)$ in \textbf{Algorithm \ref{alg:update_topo}} is still a DAG and $Q(\widetilde{\theta})<Q(\ordopt)$. $W(|\widetilde{\theta}|)$ follows the topological sort $\pi_{ij}$, so $Q(\Theta^*_{\pi_{ij}})\leq Q(\widetilde{\theta})<Q(\ordopt)$.
\end{proof}

\subsection{Proof of Lemma \ref{lemma:example}}
\begin{proof}
For $\Z_0 = \Omega\backslash\{(i_0,j_0)\}$, $W^*(\Z_0)$ is obviously a DAG, $\truedag$ is global minimum of problem \eqref{eq:abs_notears}, then $Q(\truedag)\leq Q(W^*(\Z_0))$. $\truedag$ is also a feasible solution for problem \eqref{eq:edge_absence} with $\Z_0$, then $ Q(W^*(\Z_0))\leq Q(\truedag)$. $\truedag$ is unique by assumption, thus $\truedag = W^*(\Z_0)$. For $\Z_1 =\Omega\backslash\{(i_0,j_0),(i_1,j_1)\}$, we can use the same arguments. KKTS continues until $Z_l =  \Omega\backslash\{(i_0,j_0),(i_1,j_1),\ldots,(i_l,j_l)\}$ can not guarantee the solution $W^*(\Z_l)$ to be a DAG. For example, if 
\begin{align*}
    \Z_{l-1} =& \Omega \backslash \{(i,j)| i<j, i =1,\ldots ,d\}\\
    \Z_{l} =& \Z_{l-1}\backslash \{(i_m,j_m)\}
\end{align*}
The only requirement for $(i_m,j_m)$ is $i_m>j_m$. Followed by the same argument, we know $W^*(\Z_{l-1}) = \truedag$. Using Lemma 8 from \citet{dennis2020}, $W^*(\Z_l)$ is also a DAG, hence $Q(\truedag)\leq Q(W^*(\Z_l))$. Besides, $\truedag$ is also a feasible solution for problem \eqref{eq:edge_absence} with $\Z_l$, thus $Q(W^*(\Z_l))\leq Q(\truedag)$. $\truedag$ is unique by assumption, so $\truedag = W^*(\Z_l)$. By the same arguments, KKTS continues until an irreducible $Z_m = \{(j_0,i_0)\}\union\{(i,i)| i=1,\ldots,d\}$ is returned.
\end{proof}

\subsection{Proof of Corollary \ref{cor:emptyKKT}}
\begin{proof}
Because $\Y((\ordopt),0,0)=\emptyset$, we know for any $(i,j)$ such that $[\nabla h(W(|\ordopt|))]_{ij}=0$, we have $\frac{\partial Q(\Ordopt)}{\partial \theta_{ij}} = 0$, (ii) in Lemma \ref{lemma:characterization_kkt} is satisfied. Therefore, we only need prove for any 
$(i,j)$ such that $[\nabla h(W(|\ordopt|))]_{ij}>0$, then $(\ordopt)_{ij}=0$, i.e. $[W(|\ordopt|)]_{ij} = 0$. Because $[\nabla h(W(|\ordopt|))]_{ij}>0$ implies there exist a directed walk from $j$ to $i$, which means node $j$ appear before node $i$ in topological sort, so $\theta_{ij} = 0$. Thus, (i) in Lemma \ref{lemma:characterization_kkt} is also satisfied. The explanation given at the start of the Section \ref{sec:proofs} fulfills condition (iii). (iv) is satisfied naturally by our construction. Therefore, $\Ordopt$ is a KKT point by Lemma \ref{lemma:characterization_kkt}.

\end{proof}

\subsection{Proof of Corollary \ref{cor:sufficient}}
\begin{proof}
Follows from Lemma \ref{lemma:example}.
\end{proof}

\subsection{Proof of Theorem \ref{thm:score_decrease}}
\begin{proof}
For any $p<q$, $\left[\nabla h(W(|\ordopt|))\right]_{\pi(q),\pi(p)}>0$ by definition of connected estimator. Because $\pi(p)$ appears before $\pi(q)$ in the topological sort, $[W(|\ordopt|)]_{\pi(q),\pi(p)} = 0$, i.e., $\left(\ordopt\right)_{\pi(q),\pi(p)}=0$. All pairs $(\pi(q),\pi(p))$ for $p<q$ satisfies Lemma \ref{lemma:characterization_kkt} condition (i). By the same argument from proof of corollary \ref{cor:emptyKKT}, all pairs $(\pi(p),\pi(q))$ for $p<q$ satisfies Lemma \ref{lemma:characterization_kkt} condition (ii). Condition (iii) is satisfied by the reasoning presented at the beginning of Section \ref{sec:proofs}. (iv) is satisfied naturally by our construction. Therefore, 
$\Ordopt$ is KKT point, by Theorem~\ref{theorem:kkt_local}, it is also a local minimum if $Q$ is convex. Under the connected estimator assumption, the solution at each iteration is a local minimum if $Q$ is convex.
\end{proof}

\subsection{Proof of Theorem \ref{theorem:F_sep_score_decrease}}
\begin{proof}
If $\gY(\ordopt,0,0) \not = \emptyset$, we can always construct a new topological sort $\pi_{ij}$ by Lemma \ref{lemma:always_decrease} and strictly decreases score function. Otherwise, Algorithm searches in space $\Y(\Ordopt,\tau_*,\xi^*)$ or $\Y(\Ordopt,\tau^*,\xi_*)$ to find better topological sort until it cannot. Note that at last iteration, it must be that $\mathcal{Y}(\Ordopt,0,0) = \emptyset$, such $\ordopt$ is KKT point, i.e. local minimum if $Q$ is convex by Theorem \ref{theorem:local_point}.
\end{proof}

\subsection{Proof of Theorem \ref{theorem:local_point}}
\label{app:proof:local_point}
Before we jump into the proof, let us first consider the problem 
\begin{align}
\label{eq:edge_abs}
    \min_{\Theta} \quad Q(\Theta)\qquad \subjto\quad \theta_{ij}=0,\quad (i,j)\in \mathcal{Z}
\end{align}
Remember the definition $\tilde{\theta} = \Theta\setminus \theta$ .The necessary conditions of optimality for \eqref{eq:edge_abs} are 
\begin{subequations}
\label{eq:edge_abs_nec}
\begin{align}
\frac{\partial Q(\Theta)}{\partial \theta_{ij}}=&0,\quad    (i,j)\not \in \gZ \label{eq:edge_abs_nec a}\\
\theta_{ij}=&0,\quad    (i,j) \in \gZ \label{eq:edge_abs_nec b}\\
\frac{\partial Q(\Theta)}{\partial \tilde{\theta}} = &0 \label{eq:edge_abs_nec c}
\end{align}
\end{subequations}
Given a KKT point $(\theta^+,\theta^-,\Tilde{\theta})$ in \eqref{eq:kkt_con}, we can define the set
\begin{equation}
\label{eq:set_P}
    \mathcal{P}:=\{(i,j):[\nabla h(W(\theta^++\theta^-))]_{ij}>0\}
\end{equation}
Although set $\gP$ doesn't appear in Theorem \ref{theorem:local_point} explicitly, but it appears in Lemma \ref{lemma:hard_to_name} which is key to prove the Theorem \ref{theorem:local_point}.

\begin{lemma}
\label{lemma:hard_to_name}
If $(\theta^+,\theta^-,\Tilde{\theta})$ satisfies the KKT conditions in \eqref{eq:kkt_con}, then $\Theta^* = (\theta^*,\Tilde{\theta})$  satisfies the optimality conditions in \eqref{eq:edge_abs_nec} for $\gZ = \gP$ which is defined in \eqref{eq:set_P}, where $\theta^* = \theta^+-\theta^-$. If in addition $Q(\Theta)$ is convex, then $\Theta^*$ is a minimizer of \eqref{eq:edge_abs} for $\gZ = \gP$.
\end{lemma}

\begin{proof}[Proof of Theorem~\ref{theorem:local_point}]
Let $\theta$ be feasible solution (i.e. $W(|\theta|)$ is a DAG) to \eqref{eq:abs_notears} with $\Norm{\theta-\theta^*}_F<\epsilon$ (the Frobenius norm is used for concreteness). Since $\nabla h(W(|\theta|))$ is a continuous function of $\theta$, there exists a sufficiently small $\epsilon>0$ such that $[\nabla h(W(|\theta|))]_{ij}>0$ whenever $[\nabla h(W(|\theta^*|))]_{ij}>0$, in other words for $(i,j)$ in the set $\mathcal{P}$. Then for feasible $\theta$ within such an $\epsilon$-ball around $\theta^*$, it follows from the same argument in proof of Lemma \ref{lemma:hard_to_name}, $\theta_{ij}=0$ for $(i,j)\in \gP$. $\theta$ is therefore a feasible solution to \eqref{eq:edge_abs} for $\gZ=\gP$. By Lemma \ref{lemma:hard_to_name} and the convexity of $Q$, we then have $Q(\Theta^*)\leq Q(\Theta)$ for all feasible $\theta$ such that $\Norm{\theta-\theta^*}_F<\epsilon$.
\end{proof}

\subsection{Proof of Lemma \ref{lemma:hard_to_name} }
\begin{proof}
For $(i,j)\not \in \mathcal{Z} = \gP$, we have $[\nabla h(W(|\theta|))]_{ij}=0$, because $\frac{\partial h(W(\theta^++\theta^-))}{\partial \theta_{ij}^\pm} = [\nabla h(W(|\theta|))]_{ij}\mathbf{1}$, then $\frac{\partial h(W(\theta^++\theta^-))}{\partial \theta_{ij}^\pm}=\mathbf{0}$. From \eqref{eq:kkt_con a} and \eqref{eq:kkt_con aa},
\begin{align*}
    \frac{\partial Q(\Theta)}{\partial \theta_{ij}^+} = M_{ij}^+\geq 0\quad -\frac{\partial Q(\Theta)}{\partial \theta_{ij}^-} = M_{ij}^-\geq 0
\end{align*}
It is also know that $\frac{\partial Q(\Theta)}{\partial \theta_{ij}^+} = \frac{\partial Q(\Theta)}{\partial \theta_{ij}^-}$, so $\frac{\partial Q(\Theta)}{\partial \theta_{ij}^\pm} = 0$, it is equivalent to $\frac{\partial Q(\Theta)}{\partial \theta_{ij}} = 0$. It means \eqref{eq:edge_abs_nec a} is satisfied.

For $(i,j) \in \mathcal{Z} = \gP$, we have $[\nabla h(W(|\theta|))]_{ij}>0$. Since $(\theta^+,\theta^-)$ is feasible solution, which means $W(|\theta|)$ is a DAG. Moreover, $[\nabla h(W(|\theta|))]_{ij}>0$ indicates there is a directed path from node $j$ to $i$, then it implies there is no edge from node $i$ to node $j$. Hence, $[W(|\theta|)]_{ij} = \NormI{\theta_{ij}}=\NormI{\theta_{ij}^+}+\NormI{\theta_{ij}^-} = 0$, i.e. $\theta^+_{ij}=\theta^-_{ij} = 0$. we conclude $\theta_{ij}^* = \theta^+_{ij}- \theta^-_{ij}=0$. Now \eqref{eq:edge_abs_nec b} is satisfied. From \eqref{eq:kkt_con c}, it is obvious \eqref{eq:edge_abs_nec c} is satisfied.

Therefore, $(\theta^*,\Tilde{\theta})$ satisfies the optimality conditions in \eqref{eq:edge_abs_nec} for $\gZ = \gP$, where $\theta^* =  \theta^+-\theta^-$. If $Q(\theta)$ is convex function, conditions in \eqref{eq:edge_abs_nec} is also sufficient for optimality in \eqref{eq:edge_abs}.
\end{proof}

\section{Detailed Experiments}
\label{app:sec:experiments}

\subsection{Experimental Setting}
Here we describe the details about how to generate graphs and data for Linear SEMs with different noise distributions, fully connected graphs, logistic models and nonlinear models with neural networks. For each model, a random graph $\G$ was generated from one of two random graph models, \Erdos-\Renyi (ER) or scale-free (SF) with $kd$ edges $(k \in \{1, 2, 4\})$ on average, denoted by ER$k$ or SF$k$.

\begin{itemize}
    \item \textit{\Erdos-\Renyi}(ER), Random graphs whose edges are add independently with equal probability. We simulated models with $d,2d$ and $4d$ edges (in expectation) each, denoted by $ER1, ER2,$ and $ER4$ respectively.
    \item Scale-free network(SF). Network simulated according to the preferential attachment process \cite{barabasi1999emergence}. We simulated scale-free network with $d,2d$ and $4d$ edges and $\beta=1$, where $\beta$ is the exponent used in the preferential attachment process.
\end{itemize}

\paragraph{Linear SEMs.}
Given a random DAG $B\in \{0,1\}^{d\times d}$ from one of these two graphs, we assigned edge weights independently from Unif$([-2,-0.5]\union [0.5,2])$ to obtain a weight matrix $W \in \sR^{d\times d}$. Given $W,$ we sampled $X = W^\top X+z \in \sR^{d}$ according to the following three noise models:
\begin{itemize}
    \item \textit{Gaussian noise} with equal variance(\texttt{Gauss-EV}). $z\sim \mathcal{N}(0,I_{d\times d})$
    \item \textit{Gaussian noise} with unequal variance (\texttt{Gauss-NV}): $z_i \sim \mathcal{N}(0,\sigma^2_i),i=1,\ldots,d$ where $\sigma_i\sim \text{Unif}[1,2]$
    \item  \textit{Exponential noise} (\texttt{Exp}). $z_j\sim$ Exp$(1)$, $j=1,\ldots,d$
    \item  \textit{Gumbel noise} (\texttt{Gumbel}). $z_j\sim$ Gumbel$(0,1)$, $j=1\ldots,d$
\end{itemize}

Based on these models, we generated random datasets $\rmX\in \sR^{n\times d}$ by generating the rows \iid according to one of the models above. For each simulation, we generated $n=1000$ samples for graphs with $d \in\{ 10; 20; 50; 100\}$ nodes. 
For each dataset, we run FGS, PC, NOTEARS, KKTS with NOTEARS as initialization, TOPO with random initialization, TOPO with NOTEARS as initialization, and GOLEM-EV(equal variance), GOLEM-NV(Unequal variance). Here random initialization means a topological sort $\pi$ is randomly sampled, the solve \eqref{eq:order_opt} to obtain $\ordopt$ as initialization. Finally, a post-processing threshold of $\omega = 0.3$ is applied on $W$, following \cite{xun2018}. Since FGS outputs a CPDAG instead of a DAG or weight matrix, we orient the undirected edges favorably when making comparisons. In linear model with unequal variance Gaussian noise, the minimax concave penalty (MCP) is used to approximate $\ell_0$ penalty,
\begin{equation*}
    p(w) = \left\{
    \begin{array}{ll}
        \lambda|w|-\frac{w^2}{2\beta} & \text{if } |w|\leq \beta \lambda \\
        \frac{\beta \lambda^2}{2} & \text{otherwise}
    \end{array}
    \right.
\end{equation*}
and set $\lambda = 0.005$ and $\beta = 10$.

For TOPO, we use the least-square loss $Q(W,\rmX)=\frac{1}{2n}\|\rmX-\rmX W\|_F^2$ without any regularization for all noise type. We also use the polynomial acyclicity penalty $h(A) = \Tr((I+A/d)^d)-d$ \cite{yu2019dag} and $h(A) = -\log \det(I-A)$ \cite{bello2022}, because it is faster and more accurate than $h(A) = \Tr(e^A)-d$ \cite{xun2018}. For the choices of $\Ys,\Yl,s_0$, Table \ref{tb:hyerparameters_linear} summarizes the suggested hyerparameters. The basic idea is to increase $\Ys,\Yl,s_0$ when $d$ grows or graph get denser. 
\begin{table}[H]
\centering
\begin{tabular}{c|ccc}
\hline
\# node & $\Ys$ & $\Yl$ & $s_0$ \\ \hline
$d=10$  & 30    & 45    & 1     \\
$d=20$  & 50    & 150   & 1     \\
$d=50$  & 100   & 1000  & 10    \\
$d=100$ & 150   & 2500  & 15    \\ \hline
\end{tabular}
\caption{Suggested hyperparamters for $\Ys,\Yl,s_0$}
\label{tb:hyerparameters_linear}
\end{table}

\paragraph{Logistic Models.} Given $\G$, we assigned edge weights independently from Unif$([-2,-0.5]\union [0.5,2])$ to obtain a weight matrix $W \in \sR^{d\times d}$. Given $W$, we sample $X_j$ according to following
\[X_j = \text{Bernoulli}(\operatorname{exp} (w_j^\top X)/(1+\operatorname{exp} (w_j^\top X)))\quad j=1,\ldots,d\]
Based on these models, we generated random datasets $\rmX\in \sR^{n\times d}$ by generating the rows \iid according to one of the models above. For each simulation, we generated $n=10000$ samples for graphs with $d \in\{ 10; 20; 30; 40; 50\}$ nodes. 
For each dataset, we run FGS, PC, NOTEARS, TOPO with random initialization, TOPO with NOTEARS as initialization. We use penalized log-likelihood as score function, i.e. $$Q(f,\rmX)  = \frac{1}{n}\sum_{i=1}^d\mathbf{1}_n^\top\left(\log (\mathbf{1}_n+\text{exp}(f_i(\rmX)))-\mathbf{x}_i\circ f_i(\rmX)\right)+\lambda \NormI{W}$$
where $\lambda = 0.01$.
\paragraph{Fully Connected Graphs.}
We randomly generate a topological sort $\pi$, and generated a fully connected graph that is consistent with topological sort $\pi$. Other setting is the same as Linear SEM. Because this is a really hard problem, we increase $\Ys,\Yl,s_0$ compared to Table \ref{tb:hyerparameters_linear}.

\paragraph{Nonlinear Models with Neural Networks.}
We mainly follow the nonlinear setting in \citet{zheng2020}. Given $\G$, we simulate the SEM: 
\[X_j = f_j(X_{\text{pa}(j)})+z_j\qquad \forall j\in [d]\]
where $z_j\sim \mathcal{N}(0,1)$. Here $f_j$ is a randomly initialized MLP as described in Section \ref{sec:gradient_based_review}  

For TOPO, the score function is 
\[Q(f,\rmX) = \frac{1}{2n}\sum_{i=1}^d\NormII{\mathbf{x}_i-\hat{f}_i(\rmX)}^2\]
Here each $\hat{f}_i$ is chosen as MLP with one hidden layer of size $30$ and sigmoid activation.

\paragraph{Implementation}
The implementation details of baseline are listed below: 
\begin{itemize}
    \item FGS and PC are standard baseline for structure learning. The implementation is based on the \texttt{py-causal} package, available at \href{https://github.com/bd2kccd/py-causal}{https://github.com/bd2kccd/py-causal}. For PC algorithm, use Fisher Z test. For GES, we use \texttt{cg-bic-scores} and \texttt{maxDegree=50}.
    \item NOTEARS (NOTERAS\_MLP) was implemented using Python code: \href{https://github.com/xunzheng/notears}{https://github.com/xunzheng/notears}. Its score function is least square loss with $\ell_1$ regularization. We use default threshold $\omega= 0.3$.
    \item KKTS was implemented using Python code: \href{https://github.com/skypea/DAG_No_Fear}{https://github.com/skypea/DAG\_No\_Fear}. We allow KKTS to reverse edges in each iteration to achieve best performance.
    \item GOLEM was implemented using Python and Tensorflow code: \href{https://github.com/ignavierng/golem}{https://github.com/ignavierng/golem}. We use default parameters.
\end{itemize}
In the experiments, we use default hyperparameters for those baseline unless otherwise stated. 
\subsection{Metrics}
We evaluate the performance of each algorithm with the following three metrics:
\begin{itemize}
    \item Structure Hamming distance (SHD): A standard benchmark in the structure learning literature that counts the total number of edges additions, deletions, and reversals needed to convert the estimated graph into the true graph. For PC and GES, they all return CPDAG that may contain undirected edges, in which case we evaluate them favorably by assuming correct direction for undirected edges whenever possible. 
    \item Score: the value of least square score function.
    \item KKT: Whether solution satisfies the KKT conditions, 1 stands for Yes and 0 stands for No. Define a KKT matrix for $\theta$, denoted as $K(\theta)$.
    \[
    [K(\theta)]_{ij} = \left\{\begin{array}{cc}
        \Abs{\frac{\partial Q(\Theta)}{\partial \theta_{ij}}} & \text{if } \nabla h(W(\theta)) = 0 \\
         \Abs{W(\theta)_{ij}} & \text{if } \nabla h(W(\theta))>0
    \end{array}\right.
    \]
    \[
    \text{KKT} = \left\{\begin{array}{cc}
        1 &   \text{if } \max_{ij}{[K(\theta)]_{ij}} = 0 \\
         0 &  \text{if } \max_{ij}{[K(\theta)]_{ij}}\ne0 
    \end{array}\right.
    \]
    \item Timing: how much time the algorithm takes to run, we use it to measure the speed of the algorithms.
\end{itemize}

\subsection{Sensitivity of $\Ys,\Yl,s_0$}
\label{sec:sensitivity}
In Tables 2, 3, 4, and 5, we investigate the effect of sizes of search space and the number of searching times in larger spaces on Algorithm \ref{alg:algo1}. Here we focus on two cases: (1) Simple case: $ER1$ graphs with Gaussian noise and $d=100$. (2) Hard case: $SF4$ graphs with Guassian noise and $d=100$. 
Columns represent different $\Ys=50,150,200$. Rows represent different $\Yl=1000,2000,3000$. Blank implies algorithm has stopped at current iteration. Here we use $n_0$ to indicate how many large searches has been used. Generally speaking, for sparser graphs, using small search space and small $s_0$ are enough to return a good solution. While for denser graphs, the performance of Algorithm \ref{alg:algo1} is more sensitive to the choice of $\Ys,\Yl,s_0$.
\begin{table}[t]
    \centering
    \label{tb:senseER1SHD}
    \begin{tabular}{lcccccccccccc}\toprule
    & \multicolumn{3}{c}{$n_0=0$} & \multicolumn{3}{c}{$n_0=1$} & \multicolumn{3}{c}{$n_0=2$}  & \multicolumn{3}{c}{$n_0=3$} 
    \\\cmidrule(lr){2-4}\cmidrule(lr){5-7}\cmidrule(lr){8-10}\cmidrule(lr){11-13}
          & 50  & 150 & 200        & 50  & 150 & 200  & 50  & 150 & 200    & 50  & 150 &  200\\\midrule
    1000  & 136   & 136 & 136      & 20 & 11 & 0      & 8 & 6 &            & 5 & 0 &\\
    2000  & 136   & 136 & 136      & 19 & 11 & 0      & 8 & 6 &            & 4 & 0 &\\
    3000  & 136   & 136 & 136      & 19 & 11 & 0      & 8 &6 &             & 2 & 0 &\\\bottomrule
    \end{tabular}
    \caption{Structural Hamming Distance (SHD) for different $\Ys,\Yl,n_0$ with  $d=100$ and $n=1000$ on an Gaussian ER1 graph}
\end{table}

\begin{table}[H]
    \centering
    \resizebox{\columnwidth}{!}{
    \label{tb:senseER1loss}
     \begin{tabular}{lcccccccccccc}\toprule
    & \multicolumn{3}{c}{$n_0=0$} & \multicolumn{3}{c}{$n_0=1$} & \multicolumn{3}{c}{$n_0=2$}  & \multicolumn{3}{c}{$n_0=3$} 
    \\\cmidrule(lr){2-4}\cmidrule(lr){5-7}\cmidrule(lr){8-10}\cmidrule(lr){11-13}
          & 50  & 150 & 200    & 50  & 150 & 200 & 50  & 150 & 200    & 50  & 150 & 200 \\ \midrule
    1000  & 113.017& 113.017 & 113.017  & 49.570 &	48.291	&47.215 & 47.731 &	47.874& &47.451 &	47.219&\\
    2000   & 113.017   & 113.017 &113.017     &49.281	&48.291&	47.215 & 48.141	&47.874&
 & 47.438 &	47.219 &\\
    3000   & 113.017  & 113.017 & 113.017      &49.281&	48.291	&47.215 & 48.141 &	47.874	&
 &47.369&	47.219&	
    \\\bottomrule
    \end{tabular}
    }
    \caption{Score for different $\Ys,\Yl,n_0$ with  $d=100$ and $n=1000$ on Gaussian ER1 graphs.}
\end{table}
\begin{table}[H]
    \label{tb:senseSF4SHD}
    \centering
    \begin{tabular}{lcccccccccccc}\toprule
    & \multicolumn{3}{c}{$n_0=0$} & \multicolumn{3}{c}{$n_0=1$} & \multicolumn{3}{c}{$n_0=2$}  & \multicolumn{3}{c}{$n_0=3$} 
    \\\cmidrule(lr){2-4}\cmidrule(lr){5-7}\cmidrule(lr){8-10}\cmidrule(lr){11-13}
          & 50  & 150 & 200    & 50  & 150 & 200 & 50  & 150 & 200    & 50  & 150 & 200\\\midrule
    1000  & 776 & 405 & 322     & 672 & 244 & 144 & 568   & 185 & 143      & 295 & 58  & 143\\
    2000   &774 & 405 & 349      & 693  & 311 & 40 & 455   & 112 & 38      & 56  & 0 & 38\\
    3000   & 779 & 405 & 366      & 574 & 119 & 144 & 272   & 118 & 71      &44  &0  & 50\\\bottomrule
    \end{tabular}
    \caption{Structural Hamming Distance (SHD) for different $\Ys,\Yl,n_0$ with  $d=100$ and $n=1000$ on Gaussian SF4 graphs.}
\end{table}

\begin{table}[H]
    \centering
    \label{tb:senseSF4loss}
    \resizebox{\columnwidth}{!}{
    \begin{tabular}{lcccccccccccc}\toprule
    & \multicolumn{3}{c}{$n_0=0$} & \multicolumn{3}{c}{$n_0=1$} & \multicolumn{3}{c}{$n_0=2$}  & \multicolumn{3}{c}{$n_0=3$} 
    \\\cmidrule(lr){2-4}\cmidrule(lr){5-7}\cmidrule(lr){8-10}\cmidrule(lr){11-13}
          & 50  & 150 & 200    & 50  & 150 & 200 & 50  & 150 & 200    & 50  & 150 & 200\\\midrule
    1000  & 194.871  & 67.834 & 63.498  &162.679  &57.024  & 50.124 & 82.946&55.102&49.199     & 58.244 &	48.113&	49.198\\
    2000   & 189.561   & 67.834 & 62.848 & 157.953 &61.346  & 48.028  &83.159&	50.351	&47.800   & 47.905 &47.695&	47.799\\
    3000   & 187.662   & 67.843 & 62.79   &  106.329 & 54.097 &  51.45 & 56.241  & 49.924 & 49.70      & 47.925 & 47.694 & 47.71 \\\bottomrule
    \end{tabular}
    }
    \caption{Loss for different $\Ys,\Yl,n_0$ with $d=100$ and $n=1000$ on Gaussian SF4 graphs.}
\end{table}

\subsection{Linear Models}
\label{app:sec:linear_models}
\textbf{SHD comparisons: ER and SF graphs without FGES and PC}
\begin{figure}[H]
    \centering
    \includegraphics[width = 1 \textwidth]{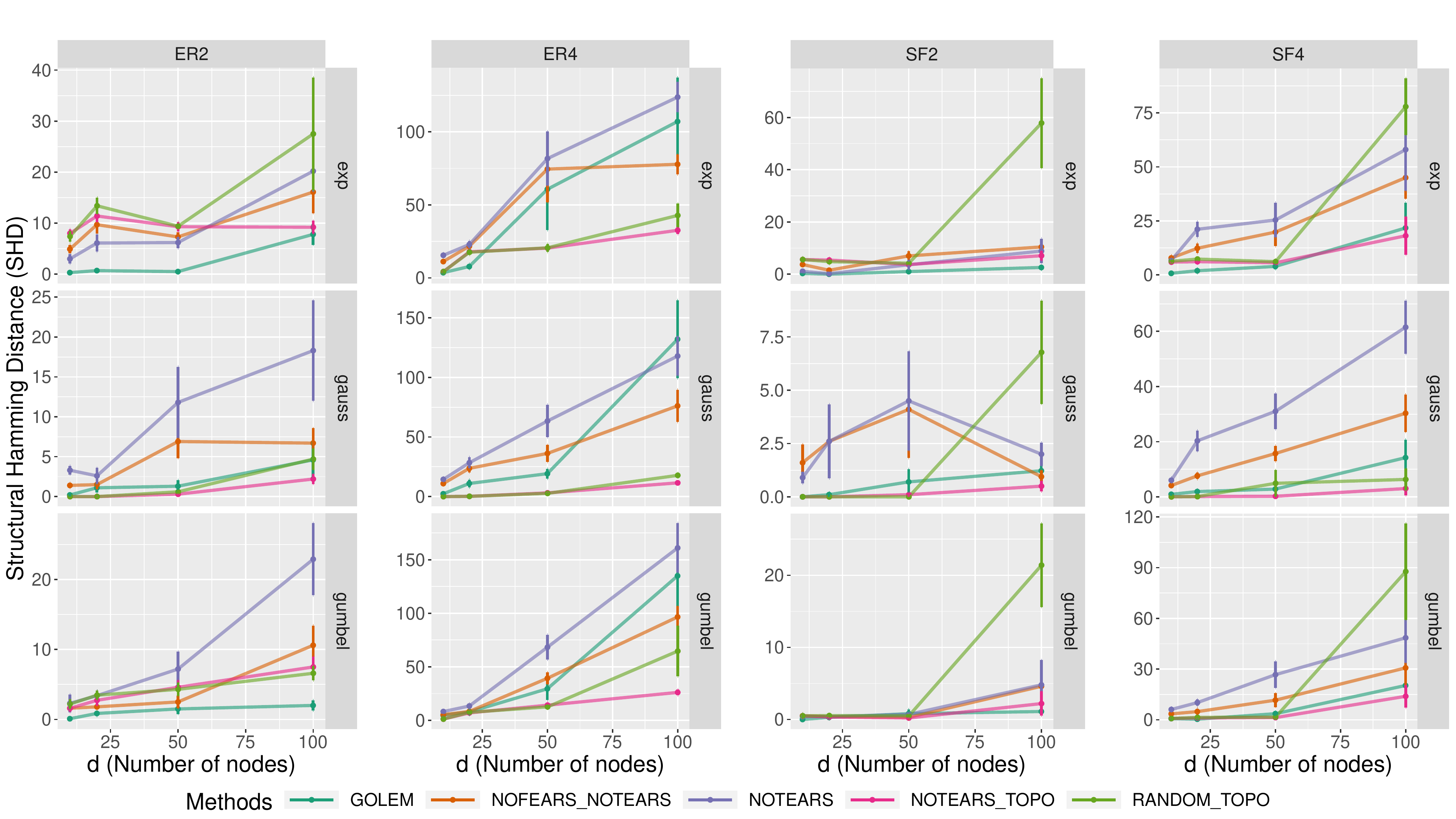}
    \caption{Stuctural Hamming distance (SHD) (lower is better). Row: noise type of SEM. Columns: random graph types, \{SF, ER\}-$k$ = \{Scale-Free,\Erdos-\Renyi\} graphs with $kd$ expected edges. Here, $\mathrm{nofears\_notears}$ (KKTS algorithm \cite{dennis2020} uses NOTEARS solution as initial point). 
	Our methods are $\topoRandom$ (random initialization), and $\topoNotears$ (using NOTEARS solution as initial point.)
	Error bars represent standard errors over 10 simulations.}
    \label{fig:shd_er_without_bad}
\end{figure}
\clearpage
\textbf{SHD comparisons: ER graphs}
\begin{figure}[H]
    \centering
    \includegraphics[width = 1 \textwidth]{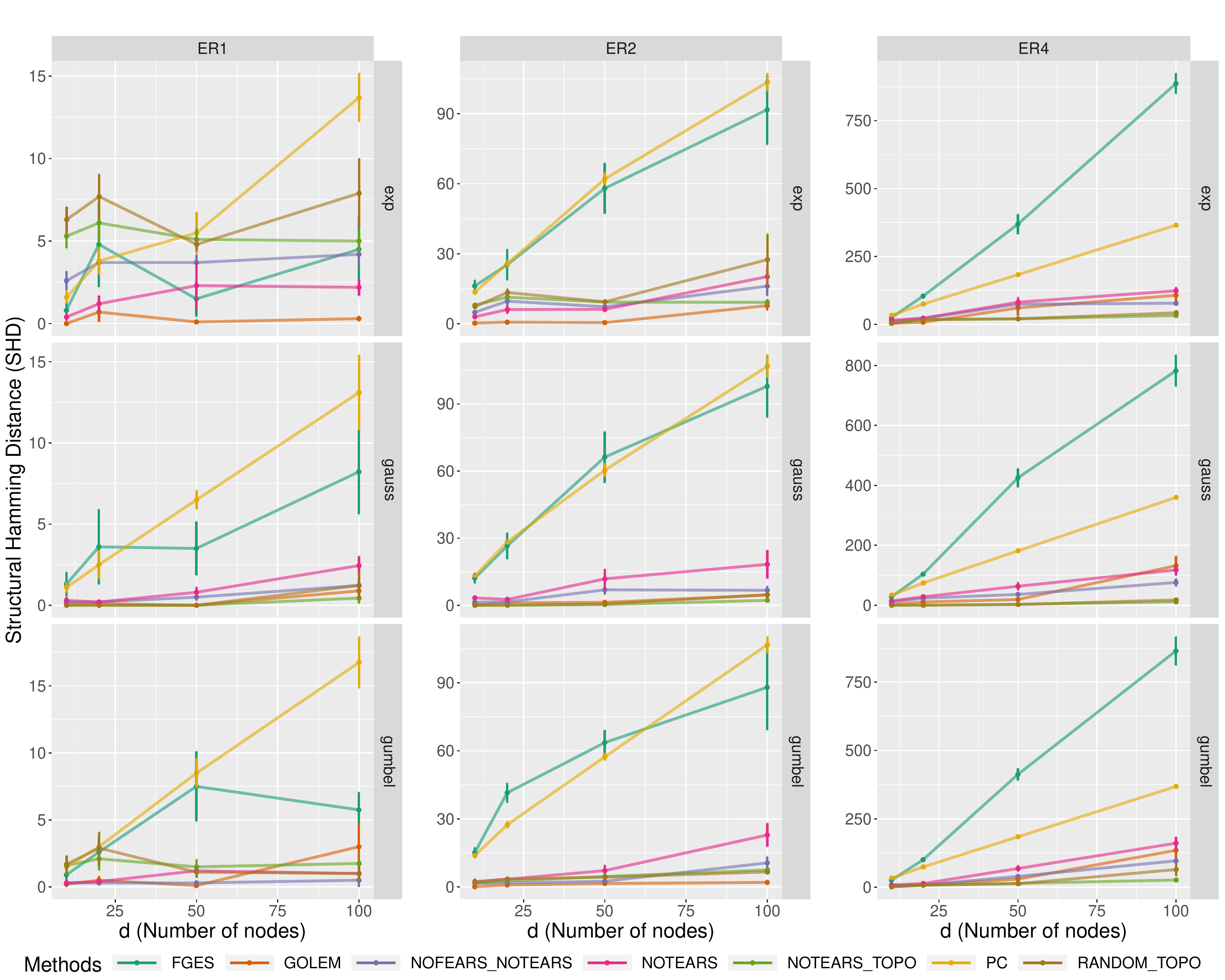}
    \caption{Stuctural Hamming distance (SHD) (lower is better). Row: noise type of SEM. Columns: random graph types, \{ER\}-$k$ = \{\Erdos-\Renyi\} graphs with $kd$ expected edges. Here, $\mathrm{nofears\_notears}$ (KKTS algorithm \cite{dennis2020} uses NOTEARS solution as initial point). 
	Our methods are $\topoRandom$ (random initialization), and $\topoNotears$ (using NOTEARS solution as initial point.)
	Error bars represent standard errors over 10 simulations.}
    \label{fig:shd_er}
\end{figure}
\clearpage
\textbf{SHD comparisons: SF graphs}
\begin{figure}[H]
    \centering
    \includegraphics[width =1 \textwidth]{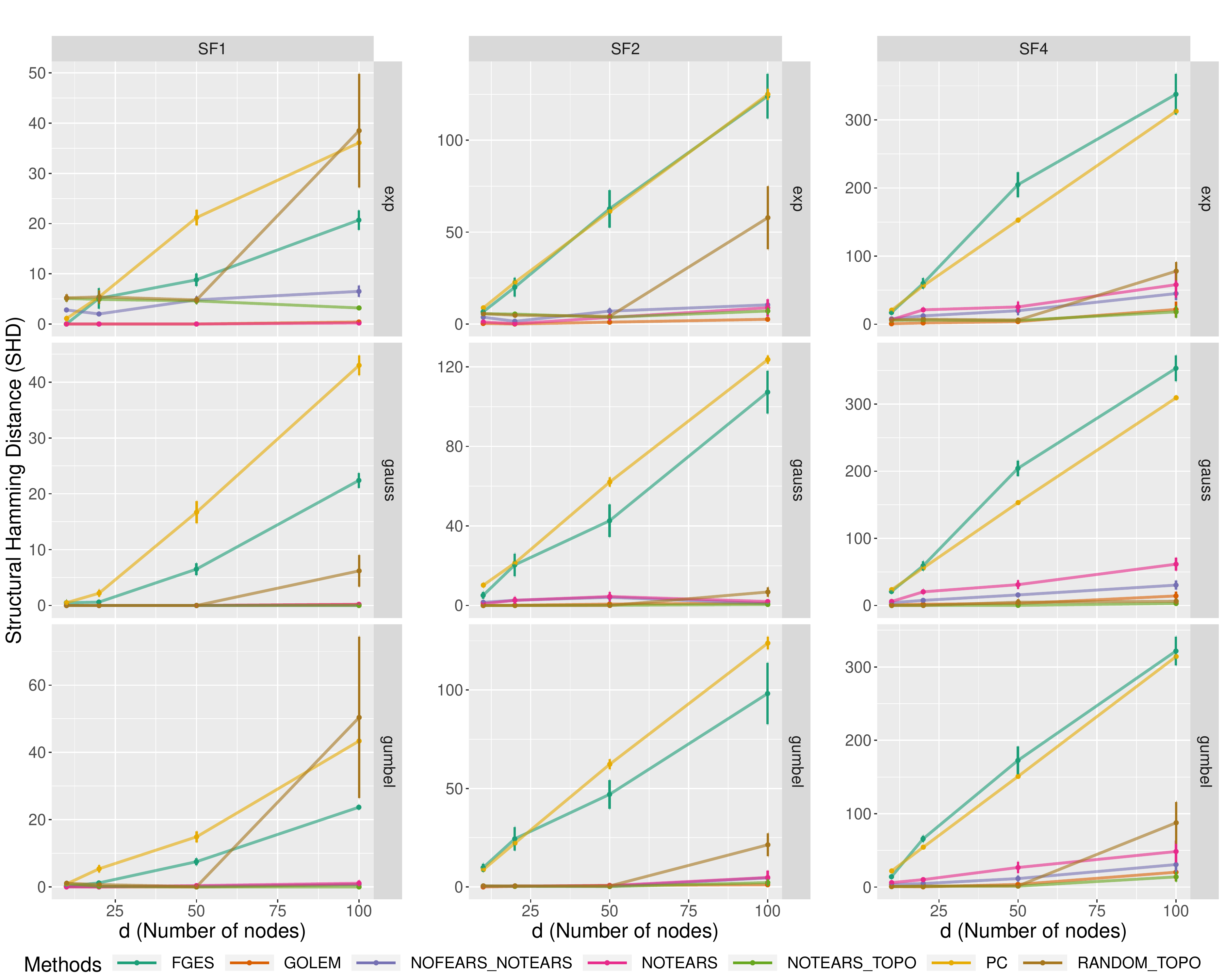}
    \caption{Structural Hamming distance(SHD) (lower is better). Row: noise type of SEM. Columns: random graph types, \{SF\}-$k$ = \{scale free\} graphs with $kd$ expected edges. Here, $\mathrm{nofears\_notears}$ (KKTS algorithm \cite{dennis2020} uses NOTEARS solution as initial point).
	Our methods are $\topoRandom$ (random initialization), and $\topoNotears$ (using NOTEARS solution as initial point.)
	Error bars represent standard errors over 10 simulations.}
    \label{fig:shd_sf}
\end{figure}
\clearpage
\textbf{Running time comparisons: ER graphs}
\begin{figure}[H]
    \centering
    \includegraphics[width =1 \textwidth]{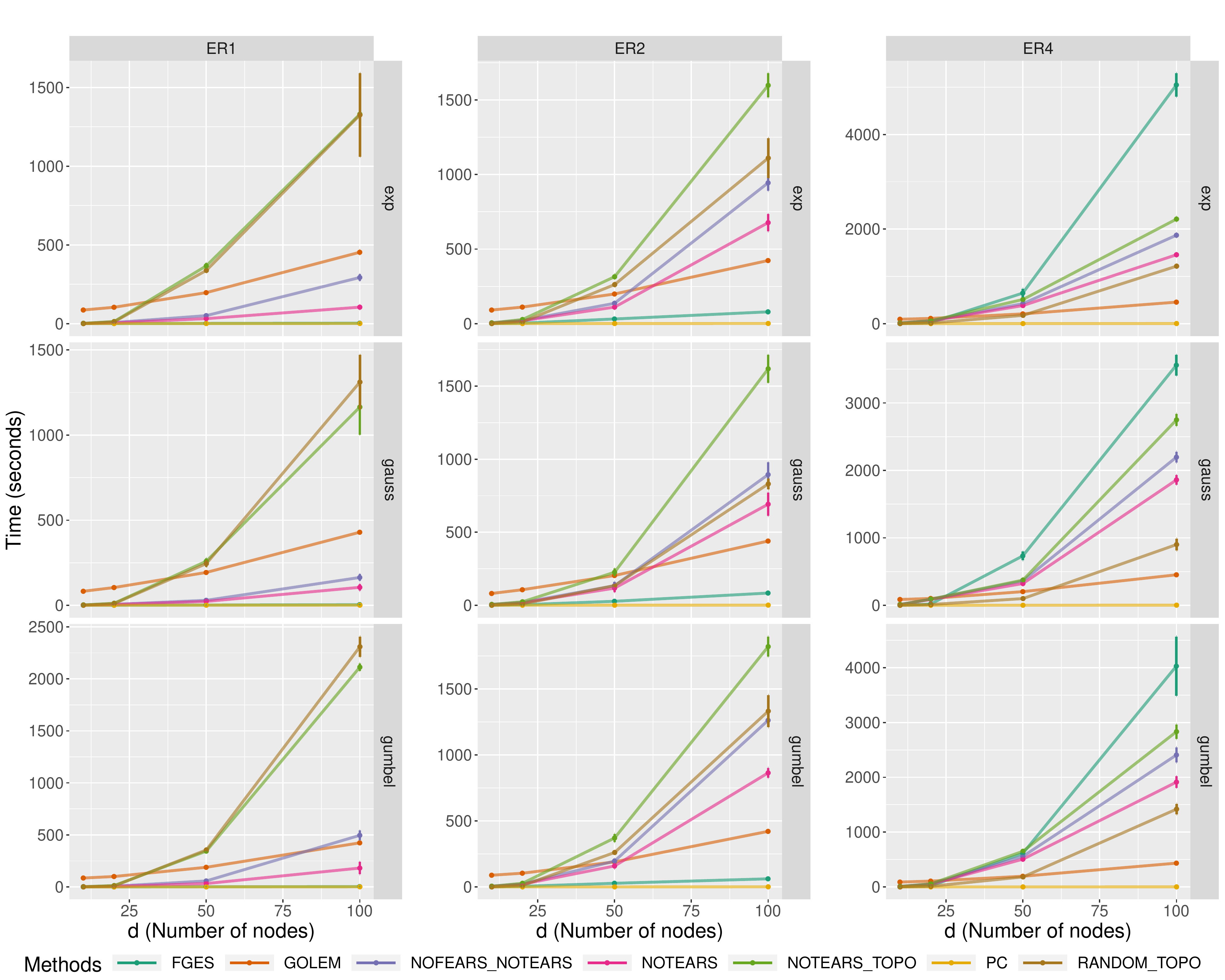}
    \caption{Runtime. Row: noise type of SEM. Columns: random graph types, \{ER\}-$k$ = \{\Erdos-\Renyi\} graphs with $kd$ expected edges. Here, $\mathrm{nofears\_notears}$ (KKTS algorithm \cite{dennis2020} uses NOTEARS solution as initial point).
	Our methods are $\topoRandom$ (random initialization), and $\topoNotears$ (using NOTEARS solution as initial point.)
	Error bars represent standard errors over 10 simulations.}
    \label{fig:time_er}
\end{figure}
\clearpage
\textbf{Running time comparisons: SF graphs}
\begin{figure}[!htb]
    \centering
    \includegraphics[width =1 \textwidth]{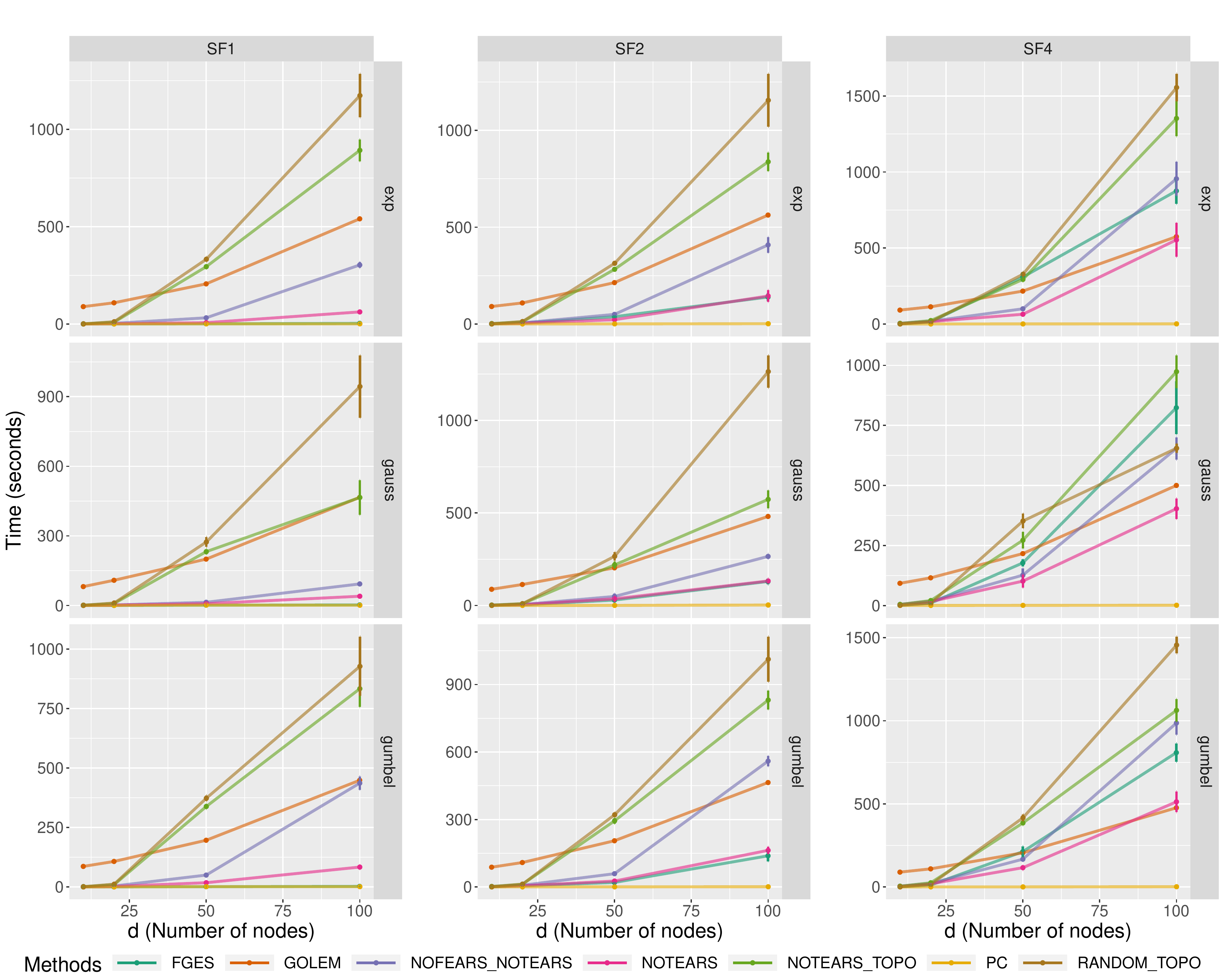}
    \caption{Structural Hamming distance(SHD) (lower is better). Row: noise type of SEM. Columns: random graph types, \{SF\}-$k$ = \{scale free\} graphs with $kd$ expected edges. Here, $\mathrm{nofears\_notears}$ (KKTS algorithm \cite{dennis2020} uses NOTEARS solution as initial point).
	Our methods are $\topoRandom$ (random initialization), and $\topoNotears$ (using NOTEARS solution as initial point.)
	Error bars represent standard errors over 10 simulations.}
    \label{fig:time_sf}
\end{figure}

\clearpage
\textbf{Score comparisons: ER graphs}
\begin{figure}[!htb]
    \centering
    \includegraphics[width = 1\textwidth]{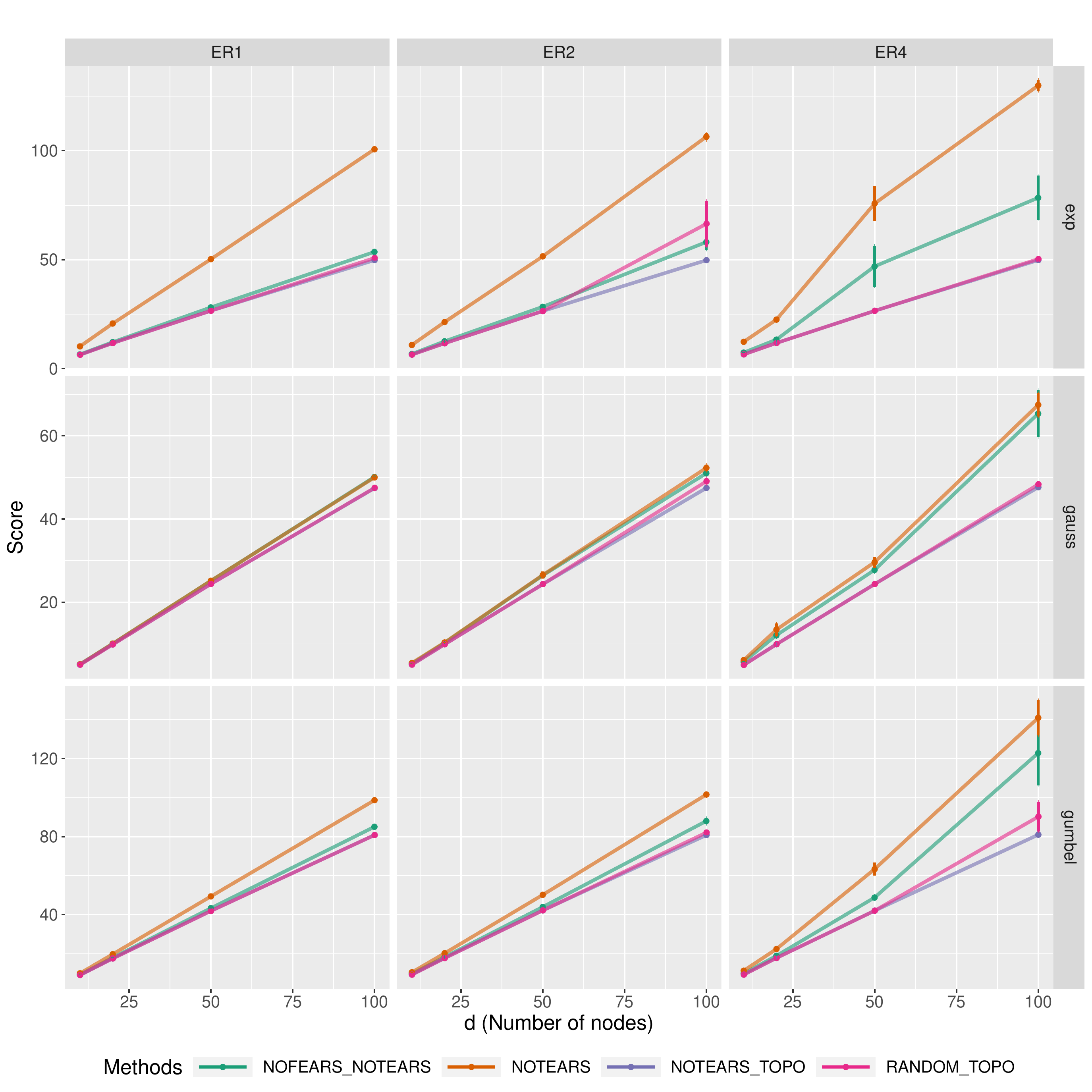}
    \caption{least square score (lower is better). Row: noise type of SEM. Columns: random graph types, \{ER\}-$k$ = \{\Erdos-\Renyi\} graphs with $kd$ expected edges. Here, $\mathrm{nofears\_notears}$ (KKTS algorithm \cite{dennis2020} uses NOTEARS solution as initial point). Our methods are $\topoRandom$ (random initialization), and $\topoNotears$ (using NOTEARS solution as initial point.)
	Error bars represent standard errors over 10 simulations.}
    \label{fig:loss_er}
\end{figure}

\clearpage
\textbf{Score comparisons: SF graphs}
\begin{figure}[H]
    \centering
    \includegraphics[width = 1\textwidth]{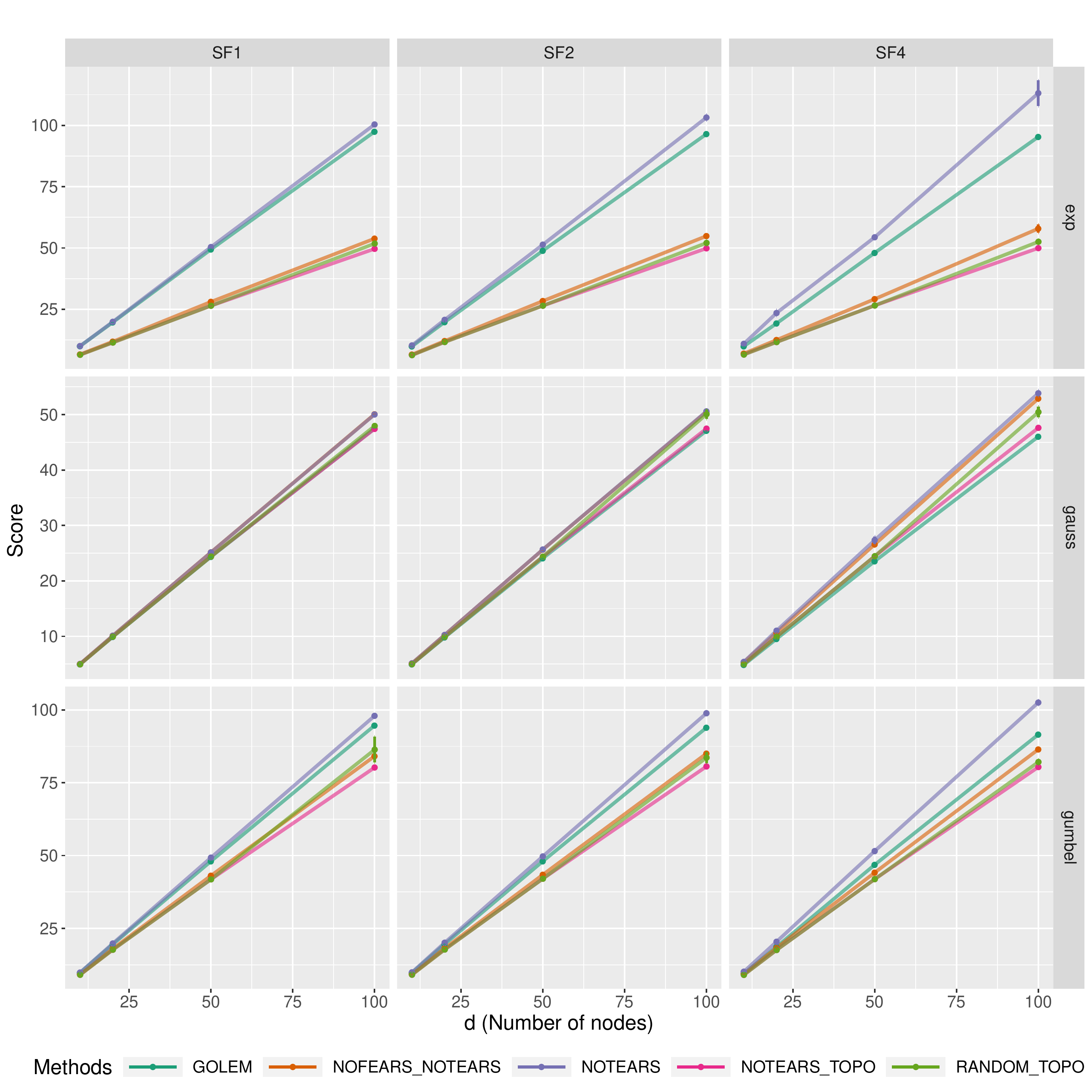}
    \caption{least square score (lower is better). Row: noise type of SEM. Columns: random graph types, \{SF\}-$k$ = \{scale free\} graphs with $kd$ expected edges. Here, $\mathrm{nofears\_notears}$ (KKTS algorithm \cite{dennis2020} uses NOTEARS solution as initial point). Our methods are $\topoRandom$ (random initialization), and $\topoNotears$ (using NOTEARS solution as initial point.)
	Error bars represent standard errors over 10 simulations.}
    \label{fig:loss_sf}
\end{figure}

\clearpage
\subsection{Nonlinear Models}
\label{app:sec:nonlinear_models}
\subsubsection{Logistic Model}
\begin{figure}[H]
    \centering
    \includegraphics[width = .75\textwidth]{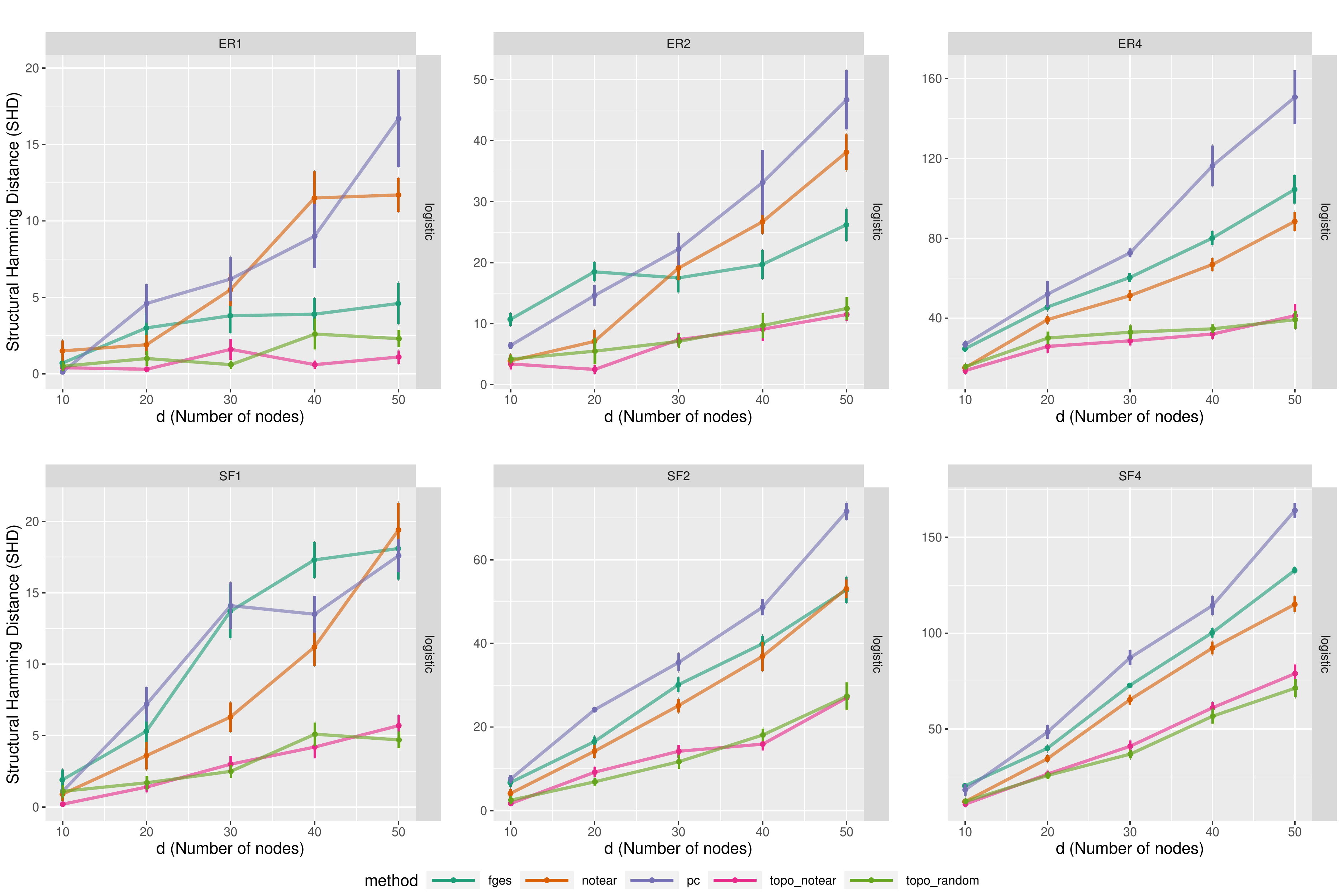}
    \caption{Structural Hamming distance(SHD) for Logistic Model, Row: random graph types, \{SF, ER\}-$k$ = \{Scale-Free,\Erdos-\Renyi\} graphs. Columns:  $kd$ expected edges.  Our methods are $\topoRandom$ (random initialization), and $\topoNotears$ (using NOTEARS solution as initial point.)
	Error bars represent standard errors over 10 simulations.}
    \label{fig:shd_logistic00}
\end{figure}
\subsubsection{Neural Networks}
\textbf{SHD comparison}
\begin{figure}[H]
    \centering
    \includegraphics[width =0.75 \textwidth]{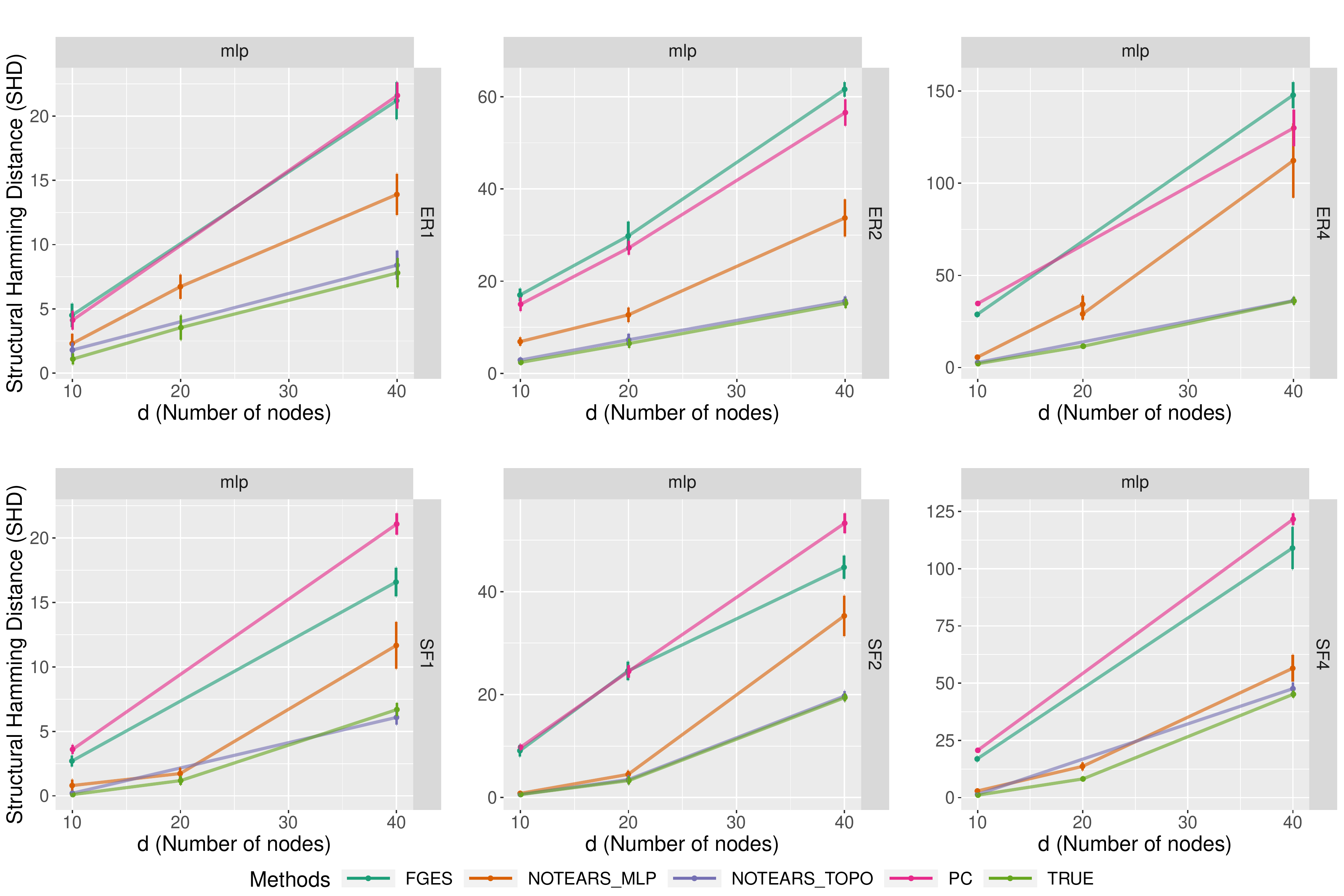}
    \caption{Structural Hamming distance(SHD) for Nonlinear Model with Neural Network, Row: random graph types, \{SF, ER\} = \{Scale-Free,\Erdos-\Renyi\} graphs. Columns:  $kd$ expected edges.  Our methods are $\topoRandom$ (random initialization), and $\topoNotears$ (using NOTEARS solution as initial point.) True(baseline): solution to \eqref{eq:order_opt_sol} with true topological sort using Neural Network.
	Error bars represent standard errors over 10 simulations.}
    \label{fig:shd_nonlinear00}
\end{figure}
\textbf{Score comparison}
\begin{figure}[H]
    \centering
    \includegraphics[width = 0.8\textwidth]{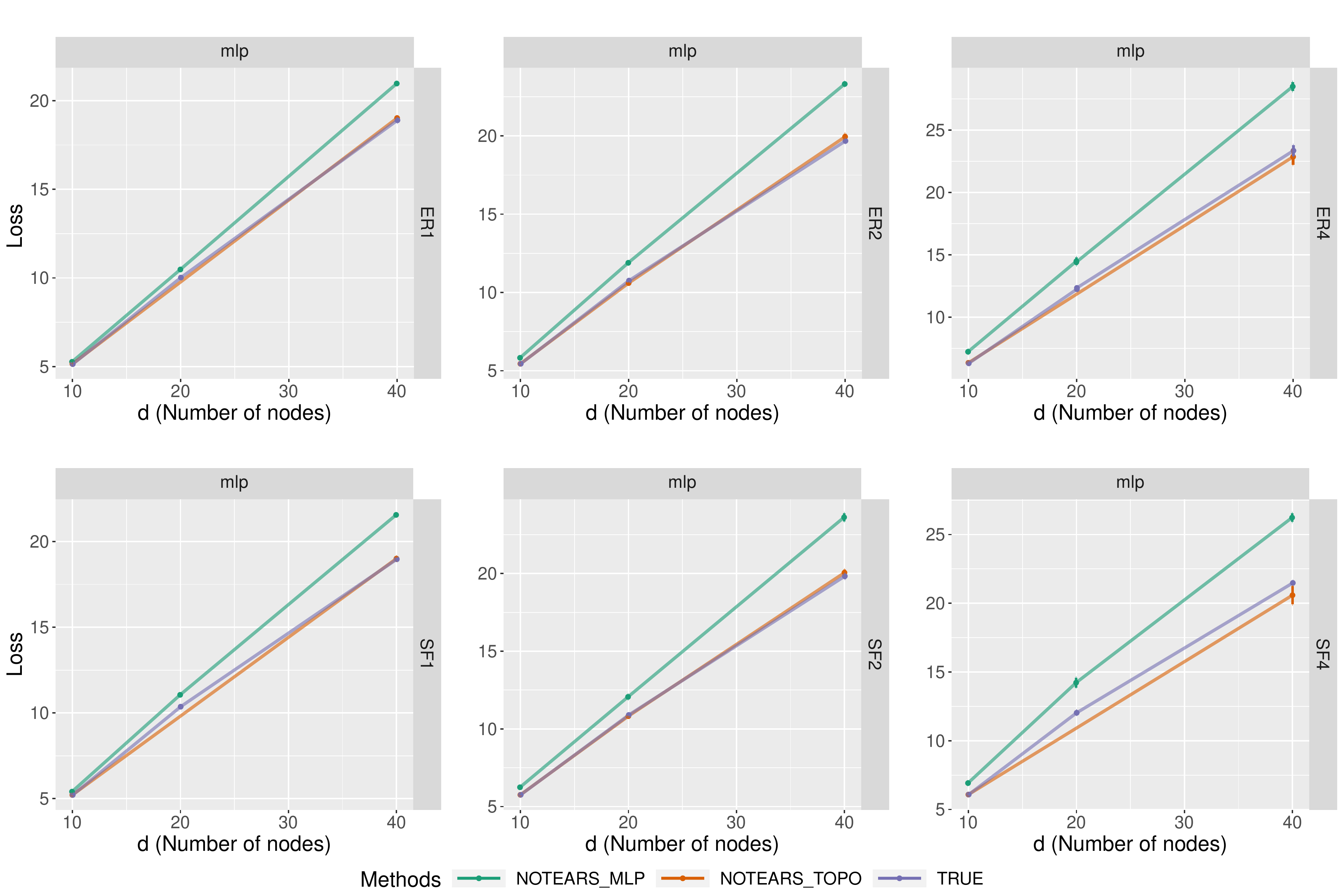}
    \caption{Score for Nonlinear Model with Neural Network, Row: random graph types, \{SF, ER\} = \{Scale-Free,\Erdos-\Renyi\} graphs. Columns:  $kd$ expected edges.  Our methods are $\topoRandom$ (random initialization), and $\topoNotears$ (using NOTEARS solution as initial point.) True(baseline): solution to \eqref{eq:order_opt_sol} with true topological sort using Neural Network.
	Error bars represent standard errors over 10 simulations.}
    \label{fig:loss_nonlinear}
\end{figure}
\subsection{Comparison against randomly chosen swapping set}
\label{app:sec:random_swapping_sets}
\begin{table}[H]
    \centering
    \begin{tabular}{ccccccc}\toprule
    & & & \multicolumn{2}{c}{TOPO} & \multicolumn{2}{c}{Random} 
    \\\cmidrule(lr){4-5}\cmidrule(lr){6-7}
       $n$   & $d$     & \# edge & SHD   & loss  & SHD   & loss \\\midrule
    1000  & 20    & 80      & 0.1   & 9.85       & 32.5   & 26.85\\
    1000  & 50    & 200     & 3     & 24.33      & 126.7   & 57.33\\
    1000  & 100   & 400     & 13.75 & 47.45      & 286.9 & 107.95\\\bottomrule
    \end{tabular}
    \caption{TOPO: the candidate swapping set $\mathcal{Y}(\theta,\tau,\xi)$ by \eqref{eq:cand_swaps} .``Random'': the TOPO algorithm chooses the candidate swapping set $\mathcal{Y}(\theta,\tau,\xi)$ randomly. Model: Linear model with Gaussian noise. Graph type: ER4 graphs. It justifies choosing swapping set $\mathcal{Y}(\theta,\tau,\xi)$ by \eqref{eq:cand_swaps} can significantly improve the performance of TOPO Algorithm.} 
    \label{Table:randomly_chosen_set}
\end{table}

\newpage
\subsection{Accuracy vs iteration}

\begin{figure}[H]
\centering
\subfigure[$d = 20$]{\label{fig:cc}\includegraphics[width=0.45\linewidth]{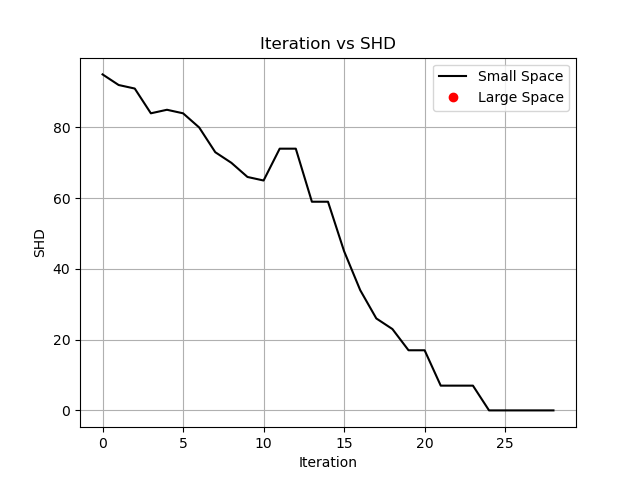}}
\subfigure[$d = 20$]{\label{fig:dd}\includegraphics[width=0.45\linewidth]{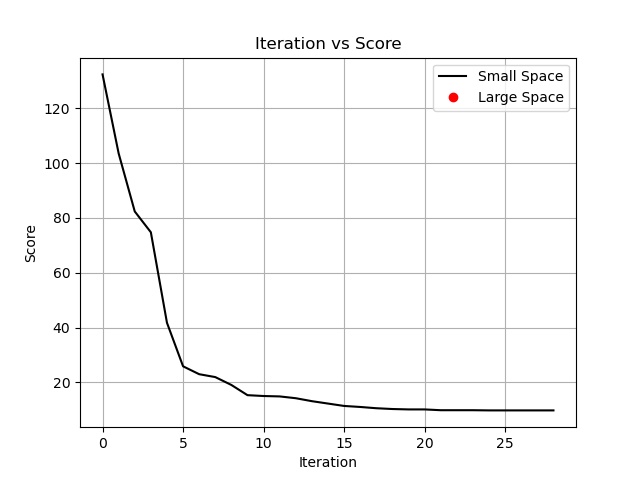}}
\\
\subfigure[$d = 50$]{\label{fig:ee}\includegraphics[width=0.45\linewidth]{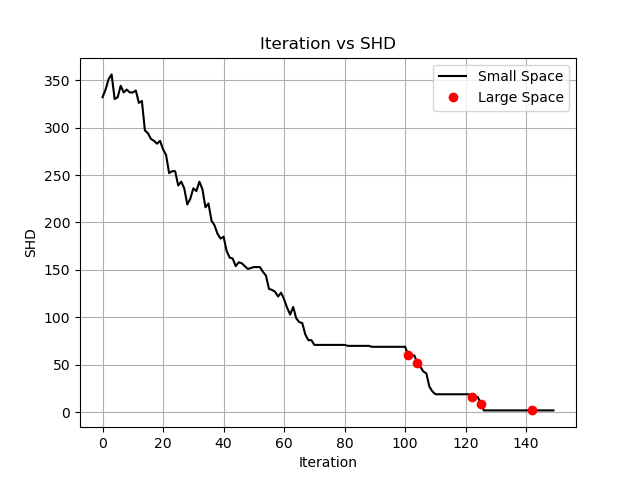}}
\subfigure[$d = 50$]{\label{fig:ff}\includegraphics[width=0.45\linewidth]{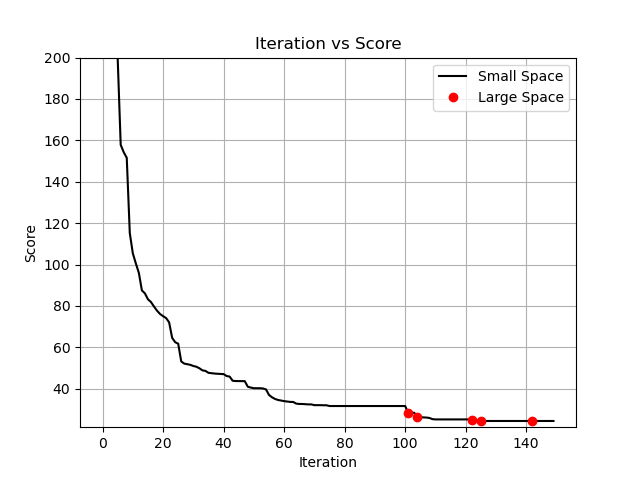}}\\
\subfigure[$d = 100$]{\label{fig:gg}\includegraphics[width=0.45\linewidth]{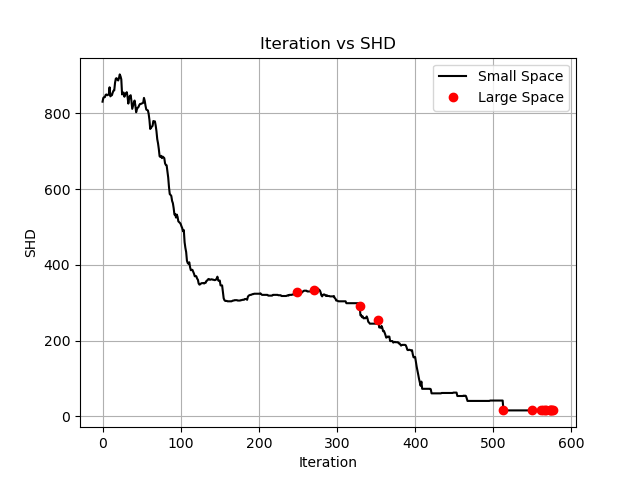}}
\subfigure[$d = 100$]{\label{fig:hh}\includegraphics[width=0.45\linewidth]{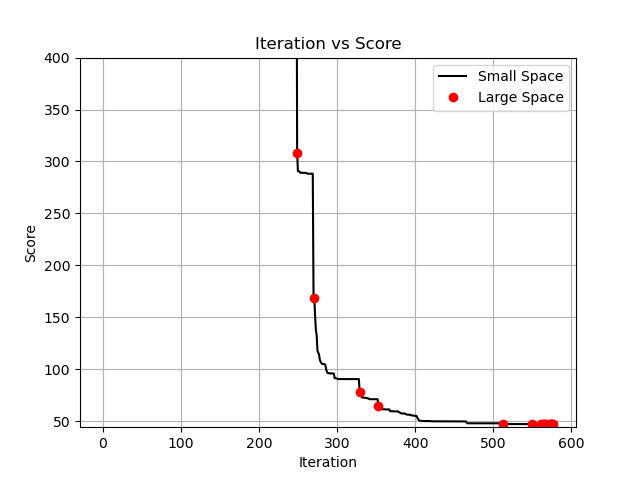}}
\caption{Iteration vs SHD (left)/Score (right). Model: linear model with Gaussian noise. Graph type: ER4 graphs. Black: search in small space. Red: search in large space. When graph is small, searching in small space is enough for finding a good local optimal, but when graph gets larger, searching in large space helps to jump out of local point and decrease the score.}
\end{figure}
\label{app:sec:accuracy_vs_iters}

\clearpage
\subsection{Greedy Strategy}

\begin{figure}[H]
    \centering
    \subfigure[Structural Hamming Distance (SHD) ]{\label{fig:ii}\includegraphics[width=0.7\linewidth]{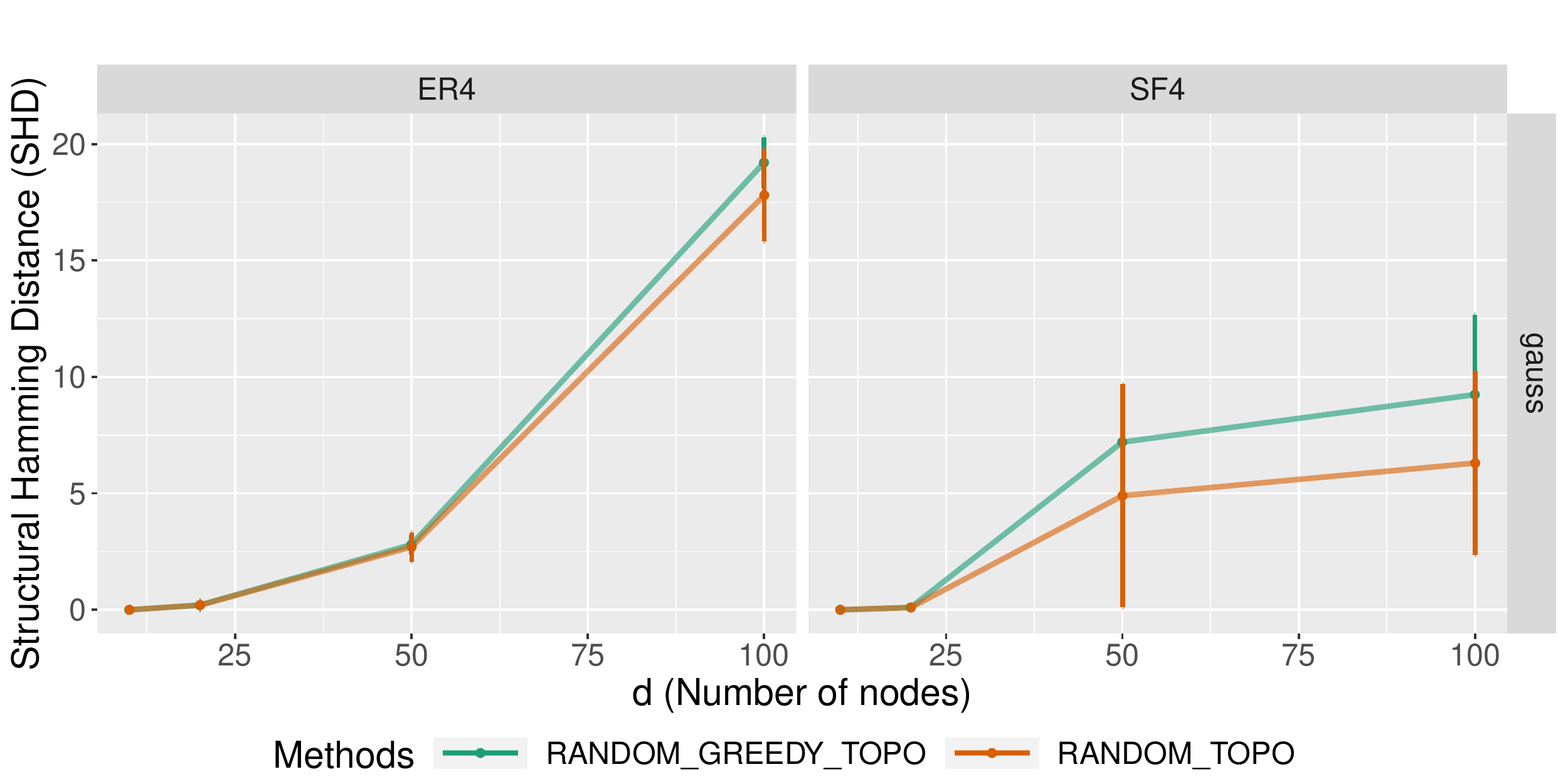}}\\
    \subfigure[Run-time (seconds)]{\label{fig:jj}\includegraphics[width=0.7\linewidth]{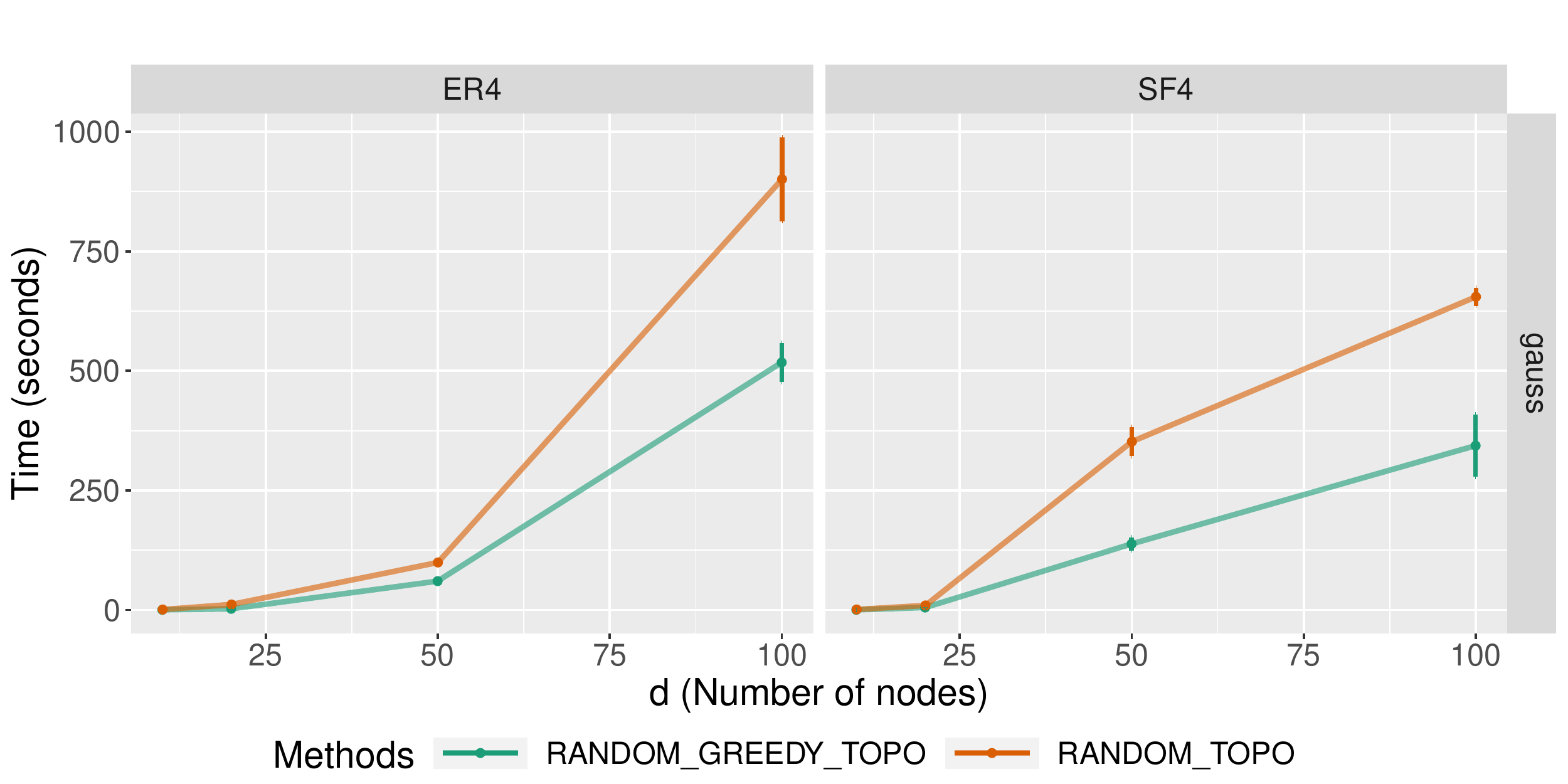}}
    \caption{Comparison between greedy scheme and non-greedy scheme by SHD \& running time.  $\topoRandom$ (TOPO starts with random initialization and uses the swap that decreases the score the most at each iteration), and $\topoGreedyRandom$ (TOPO starts with random initialization and uses the swap once it is found to decrease score.) Model: linear model with Gaussian noise. Graph type: ER4 graphs. Greedy scheme significantly improves time efficiency by sacrificing just a little accuracy. }
    \label{fig:greedy}
\end{figure}

\section{Broader Impacts}
Bayesian networks are fundamental models that represent the probabilistic relationship about how data are generated by a set of random variables. Our work contributes to the most fundamental questions: What is the underlying structure that generates data? Specifically speaking, how can one recover such structure accurately and efficiently? We propose an algorithm with theoretical guarantees to address them. The significant contribution of this work is about better solving a nonconvex continuous score-based structure learning formulation. The dramatic improvements in accuracy means better structure recovery and more accurate discovery about the underlying probabilistic relationships. 

A potential negative impact of this work is that errors in structure learning may compound into potentially more serious downstream errors. For example, a false discovery about causality may result in a company investing tons of money and efforts to remedy an incorrectly detected cause to a problem, resulting in immeasurable losses.
How to prevent incorrect causation under this continuous framework is a crucial and exciting future research direction. 

\end{document}